\documentclass{article}

\usepackage[preprint]{neurips_2024}
\usepackage{xcolor}         
\usepackage{tikz}
\usepackage{svg-extract}
\usetikzlibrary{positioning,angles,quotes,decorations.pathreplacing,calc,tikzmark}

\newcommand{\KL}{\mathrm{KL}}

\usepackage{subcaption}
\usepackage{graphicx}
\newlength{\panelheight}
\graphicspath{{./}}
\usepackage{float}
\usepackage{environ}
\usepackage{svg}
\usepackage{amsmath}
\usepackage{comment}
\usepackage{amssymb}
\usepackage{natbib} 
\everypar{\looseness=-1}
\usepackage{adjustbox}
\usepackage{booktabs}
\usepackage{tabularx}
\usepackage{mathtools}
\mathtoolsset{showonlyrefs}
\usepackage{bm}
\usepackage{amsthm}
\newtheorem*{remark}{Remark}
\newtheorem{theorem}{Theorem}[section]
\newtheorem{lemma}[theorem]{Lemma}
\newtheorem{definition}[theorem]{Definition}
\newtheorem{assumption}[theorem]{Assumption}
\newtheorem{proposition}[theorem]{Proposition}

\newtheorem*{objective@thm}{Objective}
\makeatletter
\newenvironment{objective}
  {\phantomsection\def\@currentlabel{Objective}\begin{objective@thm}}
  {\end{objective@thm}}
\makeatother

\usepackage{algorithm,algpseudocode}
\algrenewcommand\algorithmiccomment[1]{\hfill \textcolor{blue}{// #1}}

\usepackage{mathdots}
\usepackage{multicol}

\newcommand{\cnt}{\text{cnt}}

\newcommand{\mc}{\mathcal}
\newcommand{\dtv}[2]{\left\|#1,\, #2\right\|_{\text{TV}}}

\usepackage[utf8]{inputenc} 
\usepackage[T1]{fontenc}    
\usepackage{url}            
\usepackage{amsfonts}       
\usepackage{nicefrac}       
\usepackage{microtype}      
\usepackage{hyperref}       
\allowdisplaybreaks[4]

\addtocontents{toc}{\protect\setcounter{tocdepth}{-3}}

\title{Provably Safe Sim-to-Real Transfer}

\author{%
  Tingting Ni \\
  SYCAMORE, EPFL\\
  \texttt{tingting.ni@epfl.ch} \\
   \And
   Maryam Kamgarpour \\
  SYCAMORE, EPFL \\
  maryam.kamgarpour@epfl.ch\\
}

\begin{document}

\maketitle

\begin{abstract}
We address safe sim-to-real transfer, in which an agent leverages an imperfect simulator and limited real-world interaction while ensuring safety throughout data collection in the real system. This problem arises in applications such as robotics and healthcare: simulators provide cheap data, but sim-to-real mismatch makes direct transfer unreliable, and collecting real-world data to correct this mismatch must itself be safe. Moreover, deployment objectives may vary across tasks, making it costly to collect new data for each reward function. We therefore formulate safe sim-to-real transfer as a reward-free safe reinforcement learning (RL) problem, in which data are collected once and reused to plan for arbitrary reward functions. We develop a computationally efficient algorithm that identifies where the simulator and real dynamics differ, uses certified simulator transitions where they are reliable, and estimates mismatched transitions from safely collected data. With high probability, every policy deployed during learning is feasible, and the collected data support the computation of a feasible and near-optimal policy for any reward function. When the simulator is uninformative, our algorithm recovers online reward-free safe RL while improving the best-known sample complexity by a factor of \(\widetilde{\Theta}(H/\xi^2)\), where \(\xi\) is the safety margin of a baseline policy. When the simulator is accurate on most transitions, this improvement grows to \(\widetilde{\Theta}(H^2|\mc S||\mc A|/(\xi^2|\mc B|))\), where \(|\mc B|\) denotes the  size of the sim-to-real mismatch region.
\end{abstract}
\section{Introduction}
Over the last decade, reinforcement learning (RL) has achieved remarkable success in domains ranging from games and robotics to the natural sciences~\citep{mnih2013playing,silver2016mastering,ouyang2022training,lee2020learning,degrave2022magnetic}. Despite this progress, deploying RL in real world remains challenging due to two central obstacles: sample complexity and safety. Learning directly on the target system can require many interactions, each of which may be costly and subject to safety requirements, such as collision avoidance in robotics~\citep{haddadin2009requirements} or compliance with operational constraints in healthcare~\citep{kyrarini2021survey}.

A common way to reduce real-world interaction is to use a simulator, which can provide cheap samples from an approximate model of the target system. In robotics and control, simulators are often constructed from physical models, system knowledge, or prior calibration. However, the real world is difficult to model perfectly, leading to sim-to-real mismatch~\citep{tan2018sim,peng2018sim}. Consequently, policies optimized in simulation may be suboptimal in the real world. In safety-critical applications, this mismatch is even more problematic since a policy that appears safe in simulation may violate safety constraints on the real system. We model such safety requirements using constrained Markov decision processes (CMDPs), where the goal is to maximize reward while satisfying constraints. This raises the question: 
\begin{center}
    \emph{how can we use an imperfect simulator to learn a near-optimal feasible policy in the real world?\looseness-1}
\end{center}

Existing approaches address this question in different ways. A first line of work enlarges the set of environments considered during simulation training. For example, \citet{as2025spidr} study domain randomization for simulated CMDPs and optimize policies under worst-case constraints, while \citet{zhang2024distributionally} study distributionally robust constrained RL. These methods can be viewed as one-shot transfer approaches where the agent learns a policy in simulation and then deploys it on the real system. However, they can be conservative, since the learned policy optimizes worst-case or distributional performance rather than performance on the specific real system~\citep{ye2023power}. To reduce this conservatism, one can collect real-world data to fine tune the policy learned in simulation. This data-collection process must itself ensure constraint satisfaction, a requirement commonly referred to as safe exploration~\citep{koller2018learning}.\looseness-1

This motivates another line of work, known as meta RL and its safe variants, which trains agents over a distribution of simulated CMDPs. When deployed on a real system assumed to be drawn from the same distribution, these agents can fine tune using relatively few real-world interactions while ensuring safe exploration~\citep{ni2026constrained,xu2026efficient}. Among these works, \citet{ni2026constrained} show that learning a near-optimal policy under safe exploration requires only \(\tilde{\mc O}(\varepsilon^{-2}\mc C(\mc D))\) real-world samples, with a matching lower bound, where \(\mc C(\mc D)\) measures the complexity of the environment distribution \(\mc D\). This quantity is small when $\mc D$ concentrates on a low-dimensional family of environments, but can still suffer from the curse of dimensionality when distribution $\mc D$ spreads broadly across the environment space, such as uniform distribution. 

Despite these advances, both one-shot transfer methods and safe meta RL rely on a coverage assumption: the real system must be drawn from, or at least well covered by, the prescribed family or distribution of simulated CMDPs. This assumption can fail when the simulator is biased. For example, a robot may be trained on many simulated terrains, yet encounter a real floor whose wear or dust was not anticipated by the designer. Such discrepancies may be discovered only through real-world interaction and therefore cannot be incorporated into the simulator beforehand. Recent sim-to-real RL methods avoid this coverage assumption by using the simulator to guide real-world data collection, either to correct simulator mismatch~\citep{qu2025hybrid,wu2026unified} or to learn an optimal real-world policy~\citep{wagenmaker2024overcoming}. However, these methods address unconstrained settings and may therefore collect unsafe data when applied to safety-critical systems such as autonomous vehicles, surgical robots, or power grids.\looseness-1

To address safety during real-world data collection, a separate line of work studies online safe RL. Prior model-based~\citep{yu2025improved,bura2022dope,liu2021efficient,wendl2026safe} and model-free~\citep{ni2025safe} methods provide high-probability guarantees for learning a near-optimal policy under safe exploration, but they typically assume a single prespecified reward. In practice, reward functions are often iteratively engineered or vary across deployments~\citep{jin2020reward,menard2021fast}; for example, in autonomous driving, different rewards may encode different target locations. Collecting new real-world data for each reward is therefore highly sample inefficient. Reward-free safe RL~\citep{miryoosefi2022simple,huang2023safe} addresses this issue by exploring without a prespecified reward and subsequently planning for arbitrary reward functions without additional interaction with the real system.However, \citet{miryoosefi2022simple} do not ensure safe exploration during data collection, while \citet{huang2023safe} ensure safe exploration but require solving a constrained nonconvex optimization problem to compute exploration policies. Moreover, these methods are fully online and do not leverage simulator information.\looseness-1

To answer the question raised above, we address a safe sim-to-real transfer framework that leverages information from an imperfect simulator, focuses on safe real-world exploration, and supports subsequent planning for arbitrary reward functions without additional real-world interaction. We summarize our contributions below:

1. We formulate safe sim-to-real transfer as a reward-free safe RL problem, where the agent uses a simulator and limited real-world interaction under safe exploration guarantees to support accurate reward-free planning. We propose a computationally efficient algorithm that identifies reliable simulator regions, uses certified simulator transitions to reduce real-world interaction, and corrects sim-to-real mismatch with collected real-world data.

2. For our algorithm, we provide high-probability guarantees for safe real-world exploration and accurate reward-free planning; see Theorem~\ref{thm:main}. The sample-complexity bound characterizes the benefit of simulator access through the size of the mismatch region and the separation gap. When the sim-to-real mismatch is large, our framework recovers the fully online setting of \citet{huang2023safe} and improves upon its sample-complexity bound. Additional comparisons with prior work are provided in Table~\ref{tab:comparison}.\looseness-1

3.  We validate our theoretical results in a safety-critical gridworld environment. Our experiments show that the benefit of using the simulator is larger when the sim-to-real mismatch is smaller.\looseness-1
\section{Problem setting}\label{sec:problem-setting}
We first review background on constrained Markov decision processes (CMDPs), then formally state the safe sim-to-real transfer problem and its underlying assumptions.

\textbf{Notation.} Let \(\mathbb N\) and \(\mathbb R\) denote the sets of natural and real numbers, respectively. For a set \(\mathcal X\), let \(\Delta(\mathcal X)\) denote the set of probability measures on \(\mathcal X\), and let \(|\mathcal X|\) denote its cardinality. For \(p,q\in\Delta(\mathcal X)\), where \(\mathcal X\) is finite, their total variation distance is defined as
\(
\|p, q\|_{\mathrm{TV}}\coloneqq\frac{1}{2}\sum_{x\in\mathcal X}|p(x)-q(x)|.
\)
For every positive integer \(m\), let \([m]\coloneqq\{1,\ldots,m\}\). For \(z\in\mathbb R\), define its positive part as \([z]_+\coloneqq\max\{z,0\}\). For \(x>0\), define \(\log_+(x)\coloneqq\max\{\log x,0\}\). For real numbers \(x\) and \(y\), \(x\wedge y\) and \(x\vee y\) denote \(\min\{x,y\}\) and \(\max\{x,y\}\), respectively.
\subsection{Constrained Markov Decision Processes}
We consider a CMDP defined by a tuple
\(\mathcal{M} = \left(\mathcal{S}, \mathcal{A}, H, P, s_1,\mathcal{F}, c\right),\)
where $\mathcal{S}$ and $\mathcal{A}$ are finite state and action spaces, and $H$ is the horizon. The transition dynamics are given by $P:=\{P_h\}_{h=1}^{H}$, where $P_h(s'| s,a)$ denotes the probability of transitioning from state $s$ to state $s'$ after taking action $a$ at timestep $h\in[H]$. Without loss of generality, the initial state is denoted by $s_1\in\mathcal{S}$.\footnote{As explained by \citet{fiechter1994efficient}, if the initial state is instead drawn from a distribution $s_1\sim\rho$, one can equivalently introduce an artificial initial state $s_0$ such that $P_0(\cdot\mid s_0,a):=\rho(\cdot)$ for every action $a$. This augments the state space and horizon each by one, and all bounds carry over with only this constant-size modification.} In addition, we consider a class of deterministic utility functions\footnote{Although we consider deterministic utility functions for notational simplicity, our results extend to bounded stochastic utilities. }
\[
\mathcal{F} := \left\{ f=\{f_h\}_{h=1}^{H} \,\middle|\, f_h:\mathcal{S}\times\mathcal{A}\to[0,1],\forall h\in[H] \right\}.
\]
Each $f\in\mathcal{F}$ measures the performance of a policy, which we introduce next. Among these utility functions, a constraint function $c\in\mathcal{F}$ encodes the safety requirement.

A Markov policy $\pi=\{\pi_h\}_{h=1}^H$ is a collection of mappings $\pi_h:\mathcal{S}\to\Delta(\mathcal{A})$, and we let $\Pi$ denote the set of all such policies. Given a utility function $f\in\mathcal{F}$, the agent interacts with $\mathcal{M}$ as follows. Starting from the initial state $s_1$, at each timestep $h$, it selects an action $a_h\sim\pi_h(\cdot\mid s_h)$, receives utility $f_h(s_h,a_h)$, and transitions to the next state $s_{h+1}\sim P_h(\cdot\mid s_h,a_h)$. To measure the cumulative utility collected by $\pi$ under dynamics $P$, we define the state value function as
\[
V_{f,h}^{P,\pi}(s) = \mathbb{E}_{P,\pi} [\sum_{h'=h}^{H} f_{h'}(s_{h'},a_{h'})\mid s_h = s ].
\]
We call a policy $\pi$ \emph{feasible} in $\mathcal{M}$ if its constraint value satisfies $V_{c,1}^{P,\pi}(s_1) \geq \ell$, where $\ell\in[0,H]$ is a prescribed safety threshold. The set of feasible policies is defined as
\[
\Pi^{P}_{\mathrm{feas}} \coloneqq \{\pi\in\Pi \mid V_{c,1}^{P,\pi}(s_1) \geq \ell\}.
\]
And we call $\pi$ \emph{strictly feasible} if $V_{c,1}^{P,\pi}(s_1) > \ell$.
\subsection{Safe sim-to-real transfer}
In safe sim-to-real transfer, the agent has full access to a simulator during learning, but the resulting policies are ultimately deployed in the real world. We consider multiple deployment tasks whose reward functions may vary over a family \(\mathcal F\), such as reaching different target locations. In contrast, the constraint function \(c\) is fixed as all tasks operate in the same physical system. For example, safety requirements such as collision avoidance in autonomous driving or joint limits in robotic manipulation remain unchanged across tasks. The goal is to exploit cheap simulator access and limited real-world interaction to compute feasible and near-optimal policies for any reward function in \(\mathcal F\). However, two challenges arise in this safe sim-to-real transfer setting.

First, the simulator is only an approximation of the real world, and this mismatch may cause a policy that is feasible or near-optimal in simulation to become unsafe or suboptimal in the real world. Simulator access alone is therefore insufficient, and the agent must collect real-world data to identify and correct the mismatch. Second, collecting a separate real-world dataset for every reward function in \(\mathcal F\) would be highly sample inefficient.

These considerations motivate a \emph{hybrid, reward-free safe RL} formulation. In particular, \emph{hybrid} refers to a setting in which the agent leverages both simulator information and real-world interactions~\citep{song2022hybrid,xie2021policy}, whereas in online RL the agent learns solely through real-world interactions. The \emph{reward-free} perspective allows the agent to explore the environment without a prespecified reward, and later use the learned information to plan for arbitrary reward functions without further real-world interaction~\citep{miryoosefi2022simple,huang2023safe}. Compared with learning for a fixed reward, this imposes a stronger
exploration requirement because the collected data must support planning uniformly over all $r \in \mathcal F$. Combining these two perspectives, we formulate the learning problem as follows.\looseness-1

\paragraph{Interection setup.}
During learning, the agent can interact with the real world
\(
\mathcal M^{\mathrm{real}}=(\mathcal S,\mathcal A,H,P^{\mathrm{real}},s_1,\mathcal F,c)
\)
by rolling out policies. A rollout of policy \(\pi\) generates a trajectory
\(
\tau^\pi=\{s_h,a_h,s_{h+1}\}_{h=1}^{H},
\)
where \(a_h\sim\pi_h(\cdot\mid s_h)\) and
\(s_{h+1}\sim P_h^{\mathrm{real}}(\cdot\mid s_h,a_h)\).
The agent also has full knowledge of the simulator
\(
\mathcal M^{\mathrm{sim}}
=
(\mathcal S,\mathcal A,H,P^{\mathrm{sim}},s_1,\mathcal F,c),
\)
which differs from \(\mathcal M^{\mathrm{real}}\) only in its transition dynamics.

Since data collection occurs in the real world, where constraint violations such as collisions in autonomous navigation are unacceptable, every policy \(\pi\) deployed in \(\mc M^{\mathrm{real}}\) must satisfy the safety constraint. We formalize this requirement as safe exploration.

\begin{definition}[Safe exploration]\label{def:safe_exploration}
An algorithm that deploys a sequence of policies \(\{\pi^t\}_{t=1}^T\) ensures safe exploration in \(\mc M^{\mathrm{real}}\) if
\(
V_{c,1}^{P^{\mathrm{real}},\pi^t}(s_1)\geq\ell
\)
for every \(t\in[T]\).
\end{definition}
Having described the simulator and real-world interactions, we now define the objective.\looseness-1
\begin{objective}\label{objective}
Design an algorithm that ensures safe exploration in \(\mathcal M^{\mathrm{real}}\) during data collection and, for any reward function \(r\in\mathcal F\), computes without further real-world interaction a policy \(\pi\) satisfying\looseness-1
\begin{align}
\emph{(Near-optimality)}\quad
V^{P^{\mathrm{real}},\pi^\star}_{r,1}(s_1)-V^{P^{\mathrm{real}},\pi}_{r,1}(s_1)\le \epsilon\quad\text{and}\quad
\emph{(Feasibility)}\quad
\pi \in \Pi^{P^{\mathrm{real}}}_{\mathrm{feas}},
\label{eq:planning_goal}
\end{align}
where
\(
\pi^\star\in \arg\max_{\pi\in\Pi^{P^{\mathrm{real}}}_{\mathrm{feas}}}
V^{P^{\mathrm{real}},\pi}_{r,1}(s_1).
\)
\end{objective}
We say that the algorithm achieves \emph{\(\epsilon\)-planning accuracy} if both conditions above, namely near-optimality and feasibility, hold for every \(r\in\mathcal F\). The algorithm's efficiency is measured by the number of state-action pairs sampled from \(\mathcal M^{\mathrm{real}}\), which we refer to as its sample complexity.

\subsection{Assumptions}\label{sec_assumptions}
To achieve the above objective, we make two assumptions.\footnote{We discuss both assumptions in the experimental setup (see Section~\ref{sec:experiments}).} The first is a standard Slater-type condition in safe RL~\citep{bura2022dope,yu2025improved,ni2025safe,wendl2026safe}: safe exploration requires a baseline policy that is strictly feasible in \(\mc M^{\mathrm{real}}\).
\begin{assumption}[Slater's condition]\label{ass:slater}
There exist a known constant $\xi\in(0,H]$ and a known baseline policy $\pi^0\in\Pi$ such that $V_{c,1}^{P^{\mathrm{real}},\pi^0}(s_1)\ge \ell+\xi$.
\end{assumption}
The strictly positive margin \(\xi\) serves two purposes. Algorithmically, it provides a safety buffer that allows the agent to explore while maintaining feasibility. Statistically, it ensures that the empirical planning problem remains feasible once the estimation error is sufficiently small: the error cannot cause the estimated constraint value of the baseline policy \(\pi^0\) below the feasibility threshold. By contrast, if \(\pi^0\) were exactly feasible (\(\xi=0\)), even a small estimation error could make it appear infeasible, leaving the empirical problem without a certifiably feasible solution.

Next, we adopt a separation condition commonly used in hierarchical RL~\citep{chua2023provable}, meta RL~\citep{chen2022understanding,mutti2024test,brunskill2013sample}, and hybrid RL~\citep{qu2025hybrid}. It requires each transition mismatch between \(\mc M^{\mathrm{real}}\) and \(\mc M^{\mathrm{sim}}\) to be either negligible or well separated.\looseness-1
\begin{assumption}\label{ass:separation}
There exist known constants \(0\leq\epsilon_s\le\sigma_s\leq1\) such that, for every \((h,s,a)\in[H]\times\mathcal S\times\mathcal A\),
\[
\dtv{P_h^{\mathrm{real}}(\cdot\mid s,a)}{P_h^{\mathrm{sim}}(\cdot\mid s,a)}
\in[0,\epsilon_s]\cup[\sigma_s,1].
\]
\end{assumption}
This condition separates transition mismatches, measured in total variation distance, into two categories: negligible mismatches of at most \(\epsilon_s\) and significant mismatches of at least \(\sigma_s\). We refer to \(\sigma_s-\epsilon_s\) as the \emph{separation gap}. This gap allows significant mismatches to be distinguished from negligible ones using finitely many real-world samples. Prior knowledge of \(\epsilon_s\) and \(\sigma_s\) is necessary to exploit the simulator efficiently. Indeed, \citeauthor{cheung2024leveraging} show that even in multi-armed bandits, a special case of MDPs with \(H=1\), no hybrid method using an imperfect simulator is guaranteed to outperform a purely online method without prior information about the sim-to-real mismatch. \looseness-1

Under Assumption~\ref{ass:separation}, we define the mismatch region as the set of triples whose transition mismatch between the real and simulated environments is at least \(\sigma_s\).
\begin{definition}[Mismatch region]\label{def:mismatch_region}
The mismatch region $\mathcal B$ is defined as
\[
\mathcal B\coloneqq\left\{(h,s,a):
\dtv{P_h^{\mathrm{real}}(\cdot\mid s,a)}{P_h^{\mathrm{sim}}(\cdot\mid s,a)}
\ge \sigma_s\right\}.
\]
\end{definition}
When the simulator accurately models the real-world dynamics, the mismatch region \(\mathcal B\) is small, and the simulator is well calibrated on the non-mismatch region, as captured by a small \(\epsilon_s\).

\section{Algorithm Design}
\label{sec:algorithm-design}
The RF-RL algorithm of \citet{huang2023safe} achieves safe exploration and accurate reward-free planning in the fully online setting. However, directly applying RF-RL to our setting has two limitations. First, RF-RL learns solely through interactions with \(\mc M^{\mathrm{real}}\) and therefore cannot exploit simulator information to reduce real-world sample complexity. Second, building on the unconstrained reward-free RL approach of \citet{menard2021fast}, RF-RL computes its exploration policy by solving a constrained nonconvex optimization problem, which is generally intractable. Algorithm~\ref{alg:hybrid} overcomes these limitations and achieves Objective~\ref{objective}.

Algorithm~\ref{alg:hybrid}, summarized below, addresses these challenges through three components. First, Line~3 uses confidence bounds on empirical real-world transitions to progressively shrink the estimated mismatch region. Second, Line~4 constructs a hybrid model that uses empirical real-world transitions within this region and simulator transitions elsewhere. Together, these components avoid relearning dynamics that are already modeled accurately by the simulator. Third, unlike RF-RL, which maximizes a nonlinear truncated certificate over a pessimistic feasible set and therefore requires solving a nonconvex problem, Line~8 maximizes a \emph{linear} uncertainty certificate over empirically feasible policies through a standard CMDP. This produces a candidate policy that targets uncertainty relevant to reward-free planning. If the certificate $\Delta^t$ indicates that the hybrid model is sufficiently accurate, Line~10 terminates the algorithm. Otherwise, Line~13 deploys an adaptive mixture of the candidate and baseline policies to ensure real-world feasibility and reduce the remaining uncertainty. \looseness-1
\begin{algorithm}[ht]
\caption{Safe sim-to-real RL}
\label{alg:hybrid}
\begin{algorithmic}[1]
\Require Baseline policy \(\pi^0\) with margin \(\xi\), separation parameters \(\epsilon_s,\sigma_s\), stopping tolerance \(0<\tau\leq\xi/4\), simulator \(\mathcal M^{\mathrm{sim}}\), real-world access, confidence \(\delta\in(0,1)\).

\State Initialize $n_h^0(s,a)\gets0$ for all $(h,s,a)\in[H]\times\mc S\times\mc A$, and initialize $\hat{\mc B}^0\gets[H]\times\mc S\times\mc A$
\For{$t=1,2,\ldots$}
    \State Construct \(\hat{\mc B}^t\) as in Equations \eqref{eq_deletetion_rule} and \eqref{def:estimate_mismatch} \Comment{Estimating the mismatch region}
    \State Construct \(\hat P^t\) using Equations~\eqref{eq_def_hat_P_real} and~\eqref{eq_def_hybrid}\Comment{Constructing the hybrid model}
    \If{$V_{c,1}^{\hat P^t,\pi^0}(s_1)<\ell+\xi/2$}
        \State $\pi^t\gets\pi^0$
    \Else
        \State $\bar\pi^t\in\arg\max_{\pi\in\Pi_{\text{feas}}^{\hat P^t}}V_{b,1}^{\hat P^t,\pi}(s_1)$.\hfill\tikzmarknode{explore-start}{}
        \State $\Delta^t\gets\min\{H,V_{b,1}^{\hat P^t,\bar\pi^t}(s_1)\}$
        \If{$\Delta^t\leq\tau$}
            \State \textbf{break} the \textbf{for} loop
        \EndIf\hfill\tikzmarknode{explore-end}{}
        \State $\pi^t\gets\alpha^t\bar\pi^t+(1-\alpha^t)\pi^0$, where $\alpha^t$ is defined in Equation \eqref{eq_def_policy}
        \begin{tikzpicture}[remember picture,overlay]
            \draw[blue,thick,decorate,decoration={brace,amplitude=4pt}]
            ([xshift=-9em,yshift=8.8ex]explore-start.east) --
            node[midway,right=6pt,text=blue,align=left,text width=8.5em]{Computing safe exploration policy}
            ([xshift=-9em,yshift=-3.8ex]explore-end.east);
        \end{tikzpicture}
    \EndIf
    \State Roll out $\pi^t$ in $\mc M^{\mathrm{real}}$ and observe $\tau^t\coloneqq\{(s_h^t,a_h^t,s_{h+1}^t)\}_{h=1}^H$
\EndFor
\State \Return $\hat P^t$ and $\pi^{\mathrm{out}}\in\arg\max_{\pi\in\Pi}\{V_{r,1}^{\hat P^t,\pi}(s_1):V_{c,1}^{\hat P^t,\pi}(s_1)\geq\ell+\tau\}$ for any $r\in\mc F$
\end{algorithmic}
\end{algorithm}

\textbf{Estimating the mismatch region.}
We start by initializing the estimated mismatch region as the entire time-state-action space, \(\hat{\mc B}^0=[H]\times\mc S\times\mc A\), as in Line~1 of Algorithm~\ref{alg:hybrid}. In Line~3, the algorithm iteratively shrinks this region by removing triples estimated to be non-mismatch as follows. At the beginning of each iteration \(t\), given the trajectories \(\{\tau^i\}_{i=1}^{t-1}\) collected from \(\mc M^{\mathrm{real}}\) during previous iterations, we define the visitation counts for each \((h,s,a)\in[H]\times\mc S\times\mc A\) as
\[
n_h^{t}(s,a)\coloneqq\sum_{i=1}^{t-1}\mathbf{1}_{\{(s_h^i,a_h^i)=(s,a)\}}, \;\;
n_h^{t}(s,a,s')\coloneqq\sum_{i=1}^{t-1}\mathbf{1}_{\{(s_h^i,a_h^i,s_{h+1}^i)=(s,a,s')\}}.
\]
The empirical transition model estimated from real-world data is then defined as
\begin{align}
\hat P_h^{t,\mathrm{real}}(s'\mid s,a)\coloneqq\mathbf{1}_{\{n_h^{t}(s,a)>0\}}\frac{n_h^{t}(s,a,s')}{n_h^{t}(s,a)}
+
\mathbf{1}_{\{n_h^{t}(s,a)=0\}}\frac{1}{|\mc S|}.\label{eq_def_hat_P_real}
\end{align}
As shown in Lemma \ref{lem:proba_master_event}, its statistical uncertainty is quantified by the confidence bound
\[
\rho_h^t(s,a)=\min\left\{1,\sqrt{\beta(n_h^t(s,a),\delta)/(2n_h^t(s,a)\vee 1)}\right\},
\]
where \(\beta(n,\delta)\coloneqq\log(2|\mc S||\mc A|H/\delta)+|\mc S|\log(8e(n+1))\).

A triple \((h,s,a)\) is removed from the estimated mismatch region \(\hat{\mc B}^{t-1}\) when its empirical real-world transition model \(\hat P_h^{t,\mathrm{real}}(\cdot\mid s,a)\) is sufficiently close to the simulator transition model \(P_h^{\mathrm{sim}}(\cdot\mid s,a)\), after accounting for its confidence bound \(\rho_h^t(s,a)\). Specifically, define
\begin{equation}
\mc G^t\coloneqq\left\{(h,s,a)\in\hat{\mc B}^{t-1}:\dtv{\hat P_h^{t,\mathrm{real}}(\cdot\mid s,a)}{P_h^{\mathrm{sim}}(\cdot\mid s,a)}+\rho_h^t(s,a)\leq\frac{\epsilon_s+\sigma_s}{2}\right\}.\label{eq_deletetion_rule}
\end{equation}
The threshold \((\epsilon_s+\sigma_s)/2\) separates triples in the mismatch region from the remaining triples under Assumption~\ref{ass:separation}, allowing them to be distinguished once the confidence bound is sufficiently small. We then update the estimated mismatch region as
\begin{equation}\label{def:estimate_mismatch}
\hat{\mc B}^t\coloneqq\hat{\mc B}^{t-1}\setminus\mc G^t.
\end{equation}
By construction, the estimated regions are nonincreasing, namely, \(\hat{\mc B}^t\subseteq\hat{\mc B}^{t-1}\) for every \(t\). Lemma~\ref{lem:cover} further shows that, with high probability, \(\mc B\subseteq\hat{\mc B}^t\) at every iteration, so no mismatch triple is removed.

\textbf{Constructing the hybrid model.} Given the estimated mismatch region \(\hat{\mc B}^t\), we construct the hybrid model as follows: 
\begin{equation}
\hat P_h^t(\cdot\mid s,a)\coloneqq\mathbf{1}_{\{(h,s,a)\in\hat{\mc B}^t\}}\hat P_h^{t,\mathrm{real}}(\cdot\mid s,a)+\mathbf{1}_{\{(h,s,a)\notin\hat{\mc B}^t\}}P_h^{\mathrm{sim}}(\cdot\mid s,a).
\label{eq_def_hybrid}
\end{equation}
Thus, the algorithm uses empirical real-world transition estimates within the estimated mismatch region \(\hat{\mc B}^t\) and reuses the simulator transition dynamics elsewhere.

\textbf{Computing a safe exploration policy.} Given the hybrid model \(\hat P^t\), we seek a safe policy that visits the parts of the estimated mismatch region where real-world transition uncertainty affects reward-free planning. For any policy \(\pi\) and utility function \(f\in\mc F\), this uncertainty is measured by the value-estimation error \(e_{f,1}^{t,\pi}(s_1)\coloneqq|V_{f,1}^{P^{\mathrm{real}},\pi}(s_1)-V_{f,1}^{\hat P^t,\pi}(s_1)|\). The ideal exploration objective is\looseness-1 
\begin{equation}\label{eq_problem}
\max_{f\in\mc F,\,\pi\in\Pi_{\mathrm{feas}}^{P^{\mathrm{real}}}}
e_{f,1}^{t,\pi}(s_1).
\end{equation}
Optimizing this problem would identify a safe policy that collects real-world data in the regions most relevant to reducing the value-estimation error and refining the hybrid model. However, both the objective $e_{f,1}^{t,\pi}(s_1)$ and the feasible set $\Pi_{\mathrm{feas}}^{P^{\mathrm{real}}}$ depend on the unknown real-world model, so Problem~\eqref{eq_problem} cannot be solved directly.

We address this challenge in three steps. First, we construct a computable certificate that upper-bounds \(e_{f,1}^{t,\pi}(s_1)\). Second, we maximize such certificate \(e_{f,1}^{t,\pi}(s_1)\) over \(\Pi_{\mathrm{feas}}^{\hat P^t}\), which serves as a surrogate for \(\Pi_{\mathrm{feas}}^{P^{\mathrm{real}}}\). The optimizer defines a candidate exploration policy, while the optimal value provides a stopping criterion. Third, if further exploration is needed, we mix the candidate policy with the strictly feasible baseline policy \(\pi^0\) to guarantee real-world safety.\looseness-1

\emph{Step 1: Upper bound on the estimation error.}
\begin{lemma}\label{lemma:upper_bound_on_estimation_error}
Under Assumption~\ref{ass:separation}, with probability at least \(1-\delta\), simultaneously for every iteration \(t\), policy \(\pi\in\Pi\), and utility function \(f\in\mathcal F\),
\[
e_{f,1}^{t,\pi}(s_1)\leq\min\left\{H,\;V_{b,1}^{\hat P^t,\pi}(s_1)\right\}.
\]
Here, \(V_{b,h}^{\hat P^t,\pi}\) denotes the value function associated with bonus function \(b^t=\{b_h^t\}_{h=1}^H\), where \(b_h^t:\mathcal S\times\mathcal A\to[0,H]\) is defined by \(b_h^t(s,a)\coloneqq H\rho_h^t(s,a)\mathbf 1_{\{(h,s,a)\in\hat{\mathcal B}^t\}}+H\epsilon_s\mathbf 1_{\{(h,s,a)\notin\hat{\mathcal B}^t\}}\).
\end{lemma}
Within \(\hat{\mc B}^t\), the bonus \(b^t\) scales with the uncertainty in \(\hat P_h^{t,\mathrm{real}}\). Outside \(\hat{\mc B}^t\), it scales with \(\epsilon_s\), because the sim-to-real mismatch has been certified to be at most \(\epsilon_s\). This construction follows from the simulation lemma, which bounds the value-estimation error \(e_{f,1}^{t,\pi}(s_1)\) by the expected cumulative one-step bonus along trajectories generated by \(\hat P^t\) and \(\pi\). This cumulative bonus is exactly \(V_{b,1}^{\hat P^t,\pi}(s_1)\). The proof is provided in Appendix~\ref{app:upper_bound_on_estimation_error}.

Moreover, \(V_{b,1}^{\hat P^t,\pi}(s_1)\) is linear in the state-action occupancy measure, making it a tractable surrogate for Problem~\eqref{eq_problem}. This contrasts with the clipped or truncated uncertainty values used in prior work~\citep{menard2021fast,huang2023safe}.\footnote{\citet{huang2023safe} design their error certificate \(U(\pi)\), based on a truncated bonus value, which is concave function under policy mixing. They maximize this certificate subject to the empirical safety constraint \(V_c^{\hat P,\pi}\geq\ell+U(\pi)\). This defines a sublevel set of a concave function and leads to a nonconvex constrained optimization problem.\looseness-1}

\emph{Step 2: Computing a candidate policy.}
Line~5 of Algorithm~\ref{alg:hybrid} first checks whether \(\hat P^t\) provides a sufficient safety margin for \(\pi^0\), ensuring that \(\Pi_{\mathrm{feas}}^{\hat P^t}\) is nonempty. If this check fails, Line~6 deploys \(\pi^0\). Otherwise, using the certificate from Lemma~\ref{lemma:upper_bound_on_estimation_error}, Line~8 solves the following surrogate for Problem~\eqref{eq_problem}:\looseness-1
\begin{equation}\label{eq_compute_CMDP}
\bar\pi^t\in\arg\max_{\pi\in\Pi_{\mathrm{feas}}^{\hat P^t}}
V_{b,1}^{\hat P^t,\pi}(s_1).
\end{equation}
This is a standard CMDP under the known hybrid model \(\hat P^t\), with reward \(b^t\) and constraint \(c\). It can be solved using existing CMDP solvers, such as constrained policy optimization~\citep{achiam2017constrained}, or exactly through linear programming in the occupancy-measure formulation~\citep{altman2021constrained}.

The resulting certificate of Problem~\eqref{eq_compute_CMDP} is
\begin{equation}\label{eq_def_delta}
\Delta^t\coloneqq
\min\left\{H,\;V_{b,1}^{\hat P^t,\bar\pi^t}(s_1)\right\}.
\end{equation}
Because \(\bar\pi^t\) maximizes the bonus over \(\Pi_{\mathrm{feas}}^{\hat P^t}\), \(\Delta^t\) upper-bounds the uncertainty certificate of every empirically feasible policy. Thus, when \(\Delta^t\leq\tau\), the hybrid model is sufficiently accurate for reward-free planning under the tightened constraint in Line~17, as formalized by Lemma~\ref{lemma:Connection between estimation error and CMDPs}. The algorithm then terminates.\looseness-1

\emph{Step 3: Ensuring safe exploration.}
When \(\Delta^t>\tau\), further exploration is required. Although \(\bar\pi^t\) is feasible under \(\hat P^t\), Lemma~\ref{lemma:upper_bound_on_estimation_error} guarantees only that
\begin{equation}\label{eq:constraint_violation_bound}
V_{c,1}^{P^{\mathrm{real}},\bar\pi^t}(s_1)
\geq V_{c,1}^{\hat P^t,\bar\pi^t}(s_1)-\Delta^t
\geq\ell-\Delta^t.
\end{equation}
Thus, \(\bar\pi^t\) may violate the safety constraint in \(\mc M^{\mathrm{real}}\) by up to \(\Delta^t\). To compensate for this possible violation and ensure safe exploration, Line~13 executes a mixture of \(\bar\pi^t\) and the strictly feasible baseline policy \(\pi^0\), defined by
\begin{equation}
\pi^t\coloneqq\alpha^t\bar\pi^t+(1-\alpha^t)\pi^0, \; \alpha^t\coloneqq{\xi}/({\xi+\bigl(\ell+\Delta^t-V_{c,1}^{\hat P^t,\bar\pi^t}(s_1)\bigr)_+}).
\label{eq_def_policy}
\end{equation}
The mixing weight \(\alpha^t\) is chosen so that the strict feasibility margin of \(\pi^0\) offsets the possible real-world constraint violation of \(\bar\pi^t\). To implement this mixture, at the beginning of each iteration, the algorithm selects \(\bar\pi^t\) with probability \(\alpha^t\) and \(\pi^0\) otherwise, and then follows the selected policy throughout the iteration.

\section{Theoretical Guarantees}
\label{sec:theory_guarnatees}

In this section, we formalize the theoretical guarantees of our algorithm, including the real-world sample somplexity required for safe exploration and accurate planning. The main result is stated below.\looseness-1
 
\begin{theorem}\label{thm:main}
Let Assumptions~\ref{ass:slater} and \ref{ass:separation} hold with $\epsilon_s\le {\xi}/{4H^3}$. For \(\delta\in(0,1)\) and \(
\epsilon\in\left(0,\min\left\{1,\;H-{4H^4\epsilon_s}/{\xi}\right\}\right]\), run Algorithm~\ref{alg:hybrid} with \(\tau\coloneqq{\xi\epsilon}/{4H}+H^3\epsilon_s\).  Then, with probability at least \(1-\delta\), Algorithm~\ref{alg:hybrid} terminates after collecting at most
\begin{small}
\begin{equation}\label{eq:main-complexity}
\widetilde{\mc O}\!\biggl(\min\biggl\{\frac{H^7|\mc S|^2|\mc A|}{\xi^2\epsilon^2},\underbrace{\frac{H^6|\mc S||\mc B|}{\xi^2\epsilon^2}}_{(i)}+\underbrace{\frac{H^5|\mc S|\bigl(H|\mc S||\mc A|-|\mc B|\bigr)}{\xi\epsilon}\left(1+\frac{1}{(\sigma_s-\epsilon_s)H^2}\right)}_{(ii)}\biggr\}\biggr).
\end{equation}
\end{small}
samples from \(\mc M^{\mathrm{real}}\) and guarantees the following:
\begin{enumerate}
    \item \textbf{Safe exploration.} Every executed policy is feasible in \(\mc M^{\mathrm{real}}\).
    \item \textbf{Planning accuracy.} For any \(r\in\mc F\), the output policy \(\pi^{\mathrm{out}}\) is feasible and \((\epsilon+4\epsilon_sH^4/\xi)\)-optimal for the true CMDP defined by \(\max_{\pi\in\Pi_{\mathrm{feas}}^{P^{\mathrm{real}}}} V_{r,1}^{P^{\mathrm{real}},\pi}(s_1)\).
\end{enumerate}
\end{theorem}
The full proof of Theorem~\ref{thm:main} is provided in Appendix~\ref{app:proof_main}. Before presenting a proof sketch, we offer a few remarks on the sample-complexity bound and compare it with prior work.

Theorem~\ref{thm:main} gives two complementary sample-complexity bounds for Algorithm~\ref{alg:hybrid}. The first is a fully online bound that does not rely on simulator accuracy. The second captures simulator reuse: Term~(i) accounts for learning real-world dynamics on \(\mc B\), while Term~(ii) accounts for identifying where simulator transitions can be used, with a cost controlled by the separation gap \(\sigma_s-\epsilon_s\) in Assumption~\ref{ass:separation}. When the mismatch region is small, \(|\mc B|\ll H|\mc S||\mc A|\), the separation gap is large, \(\sigma_s-\epsilon_s\geq H^{-2}\), our sample complecity bound is dominated by Term~(i) and improves on the fully online bound by a factor of order \(H|\mc S||\mc A|/|\mc B|\). Since Algorithm~\ref{alg:hybrid} reuses simulator transitions outside mismatch region \(\mc B\), where the sim-to-real mismatch is bounded by \(\epsilon_s\), it incurs an additional planning error of at most \(4\epsilon_sH^4/\xi\).

RF-RL~\citep{huang2023safe} provides a fully online sample-complexity bound of \(\widetilde{\mathcal O}(|\mc A||\mc S|^2H^8/(\epsilon^2\xi^4))\) for safe exploration and \(\epsilon\)-accurate planning.\footnote{\citet{huang2023safe} state their bound in trajectories with cumulative rewards and costs bounded by \(1\). We rescale to cumulative utilities of order \(H\) and multiply by \(H\) to count the number of sampled state-action pairs.} When \(\epsilon_s=0\) and \(\mc B=[H]\times\mc S\times\mc A\), our algorithm recovers the fully online setting as it does not reuse simulator. Even then, our bound improves on RF-RL by a factor of \(H/\xi^2\). Thus, part of the sample-complexity benefit comes from the exploration-policy design rather than simulator use.

\noindent\textbf{Proof sketch.} The proof of Theorem~\ref{thm:main} proceeds in three steps. First, Lemma~\ref{lem:cover} guarantees that no true mismatch triple in \(\mc B\) is removed, while Lemma~\ref{lem:completeness} shows that each non-mismatch triple is removed after \(\widetilde{\mc O}(|\mc S|/(\sigma_s-\epsilon_s)^2)\) visits. Thus, non-mismatch triples incur only a bounded learning cost before their transition estimates are replaced by the simulator, as formalized in Lemma~\ref{lem:identification-budget}.\looseness-1

Second, Lemma~\ref{lemma:upper_bound_on_estimation_error} shows that the certificate \(\Delta^t\), defined in Equation~\eqref{eq_def_delta}, bounds reward and constraint value errors between the hybrid model $\hat P^t$ and real model $P^{\mathrm{real}}$ of the candidate policy $\bar \pi^t$. Together with the Slater margin of the baseline policy \(\pi^0\), this guarantees safety of the deployed mixture policy $\pi^t$ in \(\mc M^{\mathrm{real}}\). However, mixing can slow exploration since the candidate policy is executed only with probability \(\alpha^t\), which may be small.

Third, we show that the deployed mixture policy \(\pi^t\) continues to collect sufficient information despite mixing with the baseline policy. The key observation is \(\alpha^t\geq{\xi}/{(\xi+V_{b,1}^{\hat P^t,\bar\pi^t}(s_1))}\). Thus, a small probability \(\alpha^t\) of executing the candidate policy \(\bar\pi^t\) can occur only when its uncertainty certificate \(V_{b,1}^{\hat P^t,\bar\pi^t}(s_1)\) is large relative to the safety margin \(\xi\). This relationship ensures sufficient exploration progress even when the candidate policy is executed infrequently. We formalize this argument by relating the certificate to real-world visitation counts and separating the residual simulator error, controlled by \(\epsilon_s\), which is absorbed into the stopping tolerance \(\tau\). The resulting bound also covers iterations in which only the baseline policy is deployed. We then bound the total number of iterations by counting how often each triple can contribute to the uncertainty certificate. Mismatch triples may contribute throughout learning, producing a learning cost proportional to \(|\mc B|\). Non-mismatch triples contribute only until they are identified and removed, producing an identification cost proportional to \(H|\mc S||\mc A|-|\mc B|\) and seperation gap $\sigma_s-\epsilon_s$. Counting all triples without using deletion gives the fully online bound. Taking the smaller of these bounds yields the sample complexity in Theorem~\ref{thm:main}. At termination, \(\Delta^t\leq\tau\), and Lemma~\ref{lemma:Connection between estimation error and CMDPs} guarantees real-world feasibility and reward-free planning accuracy.\looseness-1

\section{Computational experiment}\label{sec:experiments}
We evaluate our algorithm on a \(5\times5\) gridworld with horizon \(H=12\), adapted from \citet{sutton1998reinforcement}. We compare against two baselines: (i) \emph{Reward-free safe RL}, which treats all state-action pairs as mismatch triples and does not exploit simulator information, thereby isolating the benefit of simulator access; and (ii) \emph{Unconstrained sim-to-real RL}, which exploits simulator information without enforcing safety, allowing us to quantify the cost of safe exploration. Both baselines are modifications of Algorithm~\ref{alg:hybrid} and are described in Appendix~\ref{app:exp_details}.
\paragraph{Experiment setup}
We consider a CMDP with stationary dynamics, i.e., $P_h = P$ for all $h \in [H]$ (a special case of our setting). The agent starts in the bottom-left corner (star in Figure~\ref{rf:fig:gridworld}) and has four actions: \emph{up}, \emph{right}, \emph{down}, and \emph{left}. With probability $0.85$, the intended action is executed; otherwise, the agent moves uniformly to a neighboring cell. The environment contains a central unsafe wall formed by three cells (hatched cells in Figure~\ref{rf:fig:gridworld}). The constraint function is defined as $c(s,a)=\mathbf{1}_{\{s\ \mathrm{is\ safe}\}}$, so $V_{c,1}^{P^{\mathrm{real}},\pi}(s_1)\ge \ell$ limits the expected number of unsafe visits to at most $H-\ell$. The simulator $P^{\mathrm{sim}}$ corresponds to the nominal gridworld, while the real environment introduces three \emph{windy} cells near the unsafe wall (red cells in Figure~\ref{rf:fig:gridworld}). In these cells the transition kernel is overridden by the wind: with probability $0.8$ the agent is pushed to the cell \emph{opposite} the chosen direction (e.g., \emph{up} sends it downward) instead of to the intended neighbor, while the remaining probability mass is distributed uniformly among neighboring cells. Because these windy cells are located adjacent to the unsafe wall, the resulting model mismatch is safety-critical: a policy optimized in the simulator expects to move away from the wall, but is instead pushed toward it.
\begin{figure}[t]
    \centering
    \setlength{\panelheight}{0.36\textwidth}
    \begin{subfigure}[t]{0.375\textwidth}
        \centering
        \includegraphics[height=\panelheight]{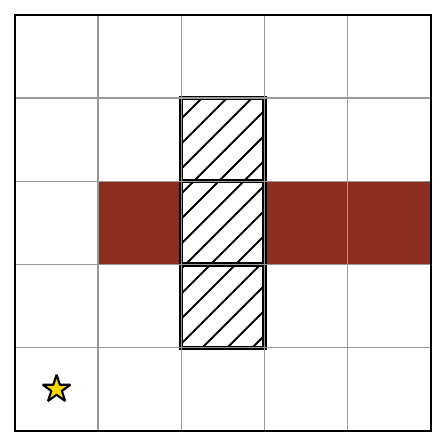}
        \caption{Gridworld instance: red cells are the mismatch region \(\mc B\), hatched cells are unsafe, and the star marks the initial state \(s_1\).}
        \label{rf:fig:gridworld}
    \end{subfigure}
    \hfill
    \begin{subfigure}[t]{0.6\textwidth}
        \centering
        \includegraphics[height=\panelheight]{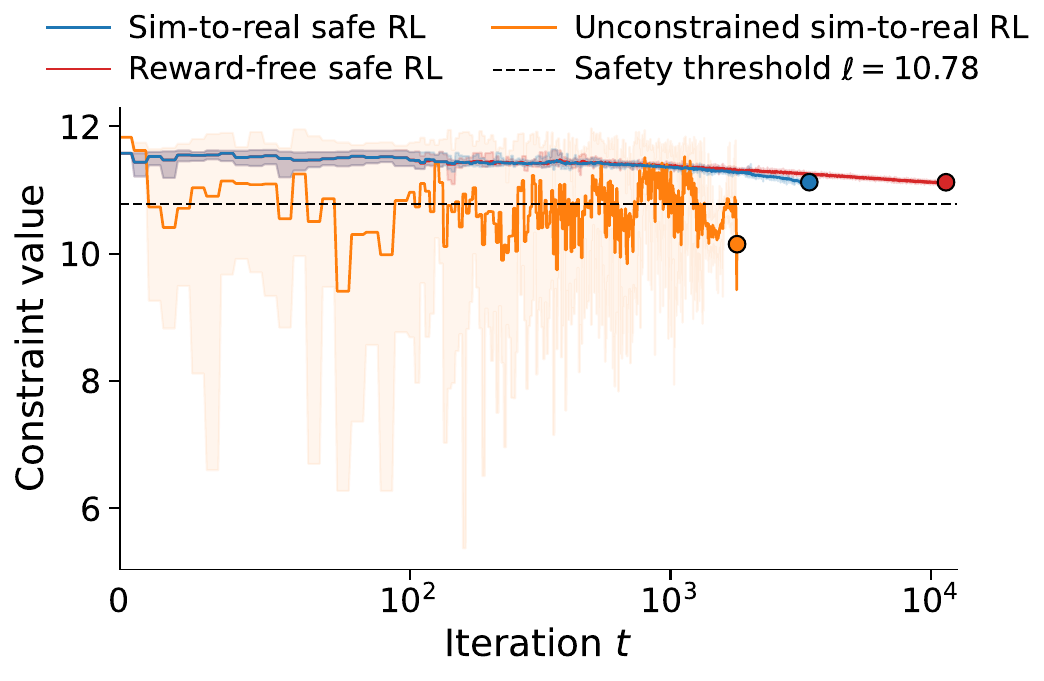}
        \caption{Constraint values of deployed policies for three algorithms over iterations. Markers denote termination points. Shaded regions indicate the min–max range across 10 seeds.}
        \label{rf:fig:safety}
    \end{subfigure}
    \label{rf:fig:exp1}
\end{figure}

\textbf{Safety versus sample complexity.}  Table~\ref{rf:tab:exp1} and Figure~\ref{rf:fig:safety} summarize the performance of all methods. All methods achieve accurate planning, but differ substantially in safety and sample complexity. The unconstrained sim-to-real baseline is the most sample-efficient, but violates the safety constraint in \(50.2\%\) of learning iterations. At the other extreme, the reward-free safe RL baseline guarantees safety throughout learning but requires \(6.4\) times as many samples from \(\mc M^{\mathrm{real}}\) compared with unconstrained sim-to-real baseline. Our method strikes a favorable balance: it maintains safe exploration while using only \(1.9\) times more samples than the unconstrained baseline. These results show that simulator information can be exploited safely to substantially reduce real-world interaction without sacrificing planning accuracy. Figure~\ref{rf:fig:heatmaps} provides further insight into how these gains arise by visualizing the state visitation distribution of each algorithm. Sim-to-real safe RL concentrates its exploration on the mismatch region, while the reward-free safe RL baseline spreads visits almost uniformly across the grid, re-learning dynamics that the simulator already models correctly. Meanwhile, sim-to-real safe RL spends less time visiting the unsafe region compared with the unconstrained baseline.

\begin{table}[t]
    \centering
\caption{Performance comparison over \(10\) seeds. Planning accuracy is defined in \eqref{eq:planning_goal} and is reported as mean \(\pm\) standard deviation over \(20\) random reward functions per seed. Unsafe iterations denote the fraction of learning iterations in which the constraint is violated.}
    \label{rf:tab:exp1}
   \resizebox{\linewidth}{!}{\begin{tabular}{lccc}
\toprule
Algorithm & Samples from $\mathcal{M}^{\mathrm{real}}$ & Suboptimality gap & Unsafe iterations \\
\midrule
Sim-to-real safe RL & $39,168 \pm 672$ & $0.003 \pm 0.008$ & $0.0\%$ \\
Reward-free safe RL & $135,456 \pm 545$ & $0.003 \pm 0.007$ & $0.0\%$ \\
Unconstrained sim-to-real RL & $21,114 \pm 276$ & $0.003 \pm 0.009$ & $50.2\%$ \\
\bottomrule
\end{tabular}
}
\end{table}
\begin{figure}[htbp]
    \centering
    \includegraphics[width=\linewidth]{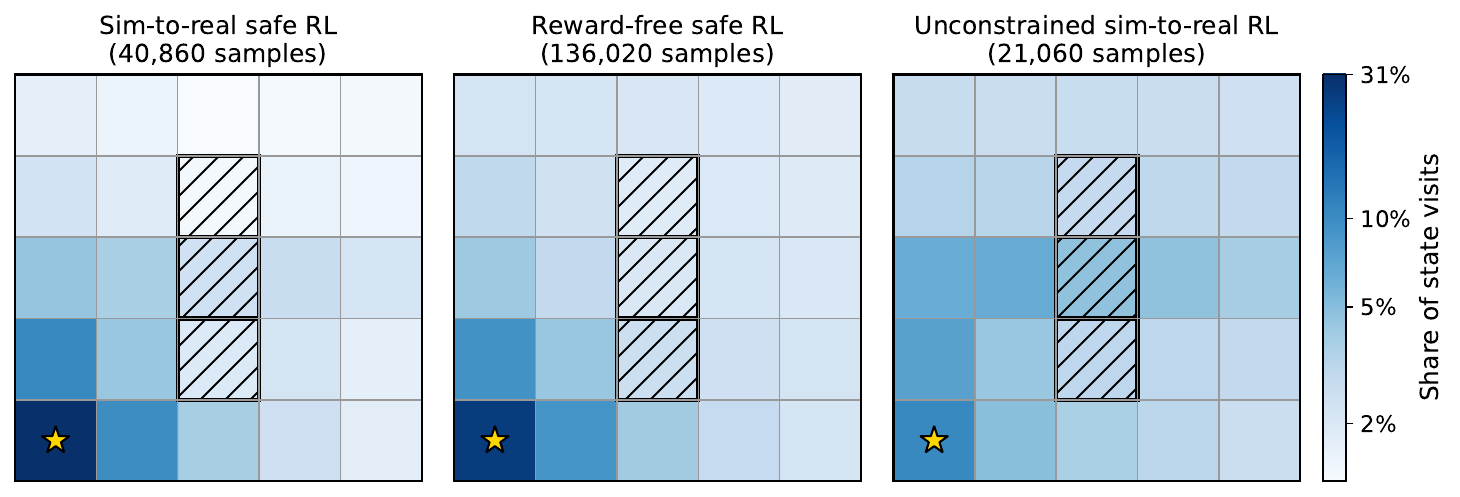}
    \caption{State visitation of the three algorithms. Each panel is normalized by the algorithm’s total number of samples, reported in the panel title, so the plots show the \emph{share} of visits falling in each cell and compare where real-world interaction is spent. The color scale is shared across panels and logarithmic. We use a logarithmic scale because \(s_1\) accounts for \(10\)--\(30\%\) of visits simply by being visited once per iteration by construction, which would otherwise dominate the rest of the grid on a linear scale.}
    \label{rf:fig:heatmaps}
\end{figure}

\textbf{Assumption verification.} For Assumption~\ref{ass:slater}, the baseline policy \(\pi^0\) can be obtained in practice through behavior cloning from expert demonstrations~\citep{torabi2018behavioral} or from a known safe controller. Its margin \(\xi\) can then be estimated by rolling out \(\pi^0\) in \(\mc M^{\mathrm{real}}\) and subtracting the safety threshold \(\ell\) from a lower confidence bound on its constraint value. Here, we manually construct \(\pi^0\) using the gridworld structure. For Assumption~\ref{ass:separation}, \(\epsilon_s\) and \(\sigma_s\) can be specified using domain knowledge or estimated from independent data~\citep{brunskill2013sample}. Here, direct computation gives \(\epsilon_s=0\) and \(\sigma_s=0.68\). When direct computation is unavailable, we provide Algorithm~\ref{alg:estimate_sigma_s} to estimate \(\sigma_s\) for a specified tolerance \(\epsilon_s\) on small modeling errors. Proposition~\ref{prop:calibration} shows that the estimate is conservative, i.e., \(\hat\sigma_s\leq\sigma_s\), when all mismatch triples are covered; otherwise, it provides a lower bound on the separation over the detected mismatch region. Appendix~\ref{app_sample_sig_s} analyzes the sample complexity of this procedure and shows that smaller values of \(\sigma_s\) require more samples.

Appendix~\ref{app_exp_sigma_s_3} studies the effect of misspecifying \(\sigma_s+\epsilon_s\), which determines the threshold in Line~3 of Algorithm~\ref{alg:hybrid}. Underestimation preserves the planning accuracy and safe exploration guarantees of Theorem~\ref{thm:main} but requires more samples. Overestimation, however, may incorrectly remove true mismatch triples from \(\hat{\mc B}^t\), reducing planning accuracy, as observed in our experiments. Such removals can also invalidate the safe exploration guarantee, although no unsafe exploration was observed experimentally. A conservative estimate of \(\sigma_s+\epsilon_s\) is therefore preferable.

Implementation details and additional ablation results on $|\mc B|$ and $\sigma_s$ are provided in Appendix~\ref{app:exp_details}.\looseness-1
\section{Conclusion and future work}
We proposed a computationally efficient algorithm for safe sim-to-real transfer that uses simulator information to guide safe real-world data collection and correct mismatched transitions. We established high-probability guarantees for safe exploration and reward-free planning. The sample-complexity bound quantifies the benefit of simulator access, particularly when the mismatch region is small and the separation gap is large. Gridworld experiments demonstrate these benefits.

\textbf{Limitations and future work.}
This work is limited to tabular CMDPs and gridworld experiments. Future directions include extending the approach to continuous spaces via function approximation, evaluating it on more realistic benchmarks, and establishing lower bounds for safe sim-to-real transfer.

\bibliographystyle{plainnat}
\bibliography{ref}
\appendix
\addtocontents{toc}{\protect\setcounter{tocdepth}{3}}
\renewcommand{\contentsname}{Organization of the Appendix}
\clearpage
\tableofcontents
\clearpage
\section{Comparison with related work}
\label{app_com}
In this section, we compare our setting with the best-known prior work along four directions: (i) \emph{sim-to-real}, whether the algorithm exploits simulator information despite mismatch with the real world or learns entirely from real-world interaction, corresponding to the fully online setting; (ii) \emph{reward-free learning}, whether the returned model supports planning for arbitrary reward functions or only the single reward used during learning; (iii) \emph{constraints}, whether constraints are part of the problem formulation; and (iv) \emph{safe exploration}, whether the algorithm guarantees that constraints are satisfied during the interaction with the real world.  The results are summarized in Table~\ref{tab:comparison}. Our method is the only one in the comparison that simultaneously handles simulator mismatch, constraints, safe real-world exploration, and reward-free planning for arbitrary rewards.To the best of our knowledge, it is the first provable algorithm with all four properties.
\begin{table}[H]
    \centering
    \caption{Comparison with related work.}
    \label{tab:comparison}
    \begingroup
    \small
    \setlength{\tabcolsep}{3pt}
    \renewcommand{\arraystretch}{1.18}
    \begin{tabularx}{\textwidth}{@{}>{\raggedright\arraybackslash}p{0.22\textwidth}@{\hspace{2pt}}>{\centering\arraybackslash}p{0.115\textwidth}>{\centering\arraybackslash}p{0.065\textwidth}>{\centering\arraybackslash}p{0.09\textwidth}>{\centering\arraybackslash}p{0.135\textwidth}@{\hspace{2pt}}>{\centering\arraybackslash}X@{}}
        \toprule
        {\scriptsize\bfseries Algorithm} & {\scriptsize\bfseries Sim-to-real} & {\scriptsize\bfseries Reward} & {\scriptsize\bfseries Constraints} & {\scriptsize\bfseries Safe exploration} & {\scriptsize\bfseries Sample complexity} \\
        \midrule
        \citet{menard2021fast} & $\times$ & any & $\times$ & $\times$ & \(\widetilde{\mc O}\!\left(\frac{H^4|\mc S|^2|\mc A|}{\epsilon^2}\right)\) \\
        \citet{miryoosefi2022simple} & $\times$ & any & $\checkmark^\dagger$ & $\times$ & \adjustbox{max width=\linewidth}{\(\displaystyle \widetilde{\mc O}\!\biggl(\frac{\min\{d,|\mc S|\}H^5|\mc S||\mc A|}{\epsilon^2}+\frac{H^4|\mc S|^2|\mc A|}{\epsilon}\biggr)\)} \\
        \citet{huang2023safe} & $\times$ & any & \checkmark & \checkmark & \(\widetilde{\mc O}\!\left(\frac{|\mc A||\mc S|^2H^8}{\epsilon^2\xi^4}\right)\) \\
        \citet{yu2025improved} & $\times$ & fixed & \checkmark & \checkmark & \(\widetilde{\mc O}\!\left(\frac{|\mc A||\mc S|^2H^6}{\epsilon^2\xi^2}\right)\) \\
        \citet{ni2025learning} & $\times$ & fixed & \checkmark & \checkmark & \(\widetilde{\mc O}\!\left(\frac{1}{\min\{\epsilon^6,\xi^6\}}\right)\) \\
        \citet{wagenmaker2024overcoming} & \checkmark & fixed & $\times$ & $\times$ & \(\widetilde{\mc O}\!\left(\frac{H^{16}}{\epsilon^8}\right)\) \\
        \citet{qu2025hybrid} & \checkmark & fixed & $\times$ & $\times$ & \adjustbox{max width=\linewidth}{\shortstack{{\(\displaystyle \widetilde{\mc O}\!\left(\min\!\left\{\frac{H^4|\mc A||\mc S|}{\epsilon^2},\frac{H^3|\mc B|}{\epsilon^2}+\frac{H^3|\mc S|^2|\mc A|}{\sigma_s^2\sigma_r^2}\right\}\right)\)}}} \\
        \citet{wu2026unified} & \checkmark & fixed & $\times$ & $\times$ & \(\widetilde{\mc O}\!\left(\frac{H^4|\mc A||\mc S|}{\epsilon^2}\right)\) \\
        \midrule
        \textbf{Our algorithm} & \checkmark & any & \checkmark & \checkmark & \adjustbox{max width=\linewidth}{\shortstack[l]{{\(\displaystyle \widetilde{\mc O}\!\biggl(\min\biggl\{\frac{H^7|\mc S|^2|\mc A|}{\xi^2\epsilon^2},{\frac{H^6|\mc S||\mc B|}{\xi^2\epsilon^2}}+{}\)}\\{\(\displaystyle {\frac{H^5|\mc S|\bigl(H|\mc S||\mc A|-|\mc B|\bigr)(1+(\sigma_s-\epsilon_s)H^2)}{\xi\epsilon(\sigma_s-\epsilon_s)H^2}}\biggr\}\biggr)\)}}} \\
        \bottomrule
    \end{tabularx}
    \par\smallskip
    \endgroup
\end{table}
\begin{remark}
\begin{itemize}
\item \citet[Theorem~4.1]{miryoosefi2022simple} require
\[
\widetilde{\mc O}\!\left(\frac{\min\{d,|\mc S|\}H^4|\mc S||\mc A|}{\epsilon^2}+\frac{H^3|\mc S|^2|\mc A|}{\epsilon}\right)
\]
iterations, where \(d\) is the dimension of the vector-valued return. To convert this result to the sample-complexity bound used in this work, we multiply the bound by \(H\) to account for the \(H\) state-action samples collected in each iteration.

However, their analysis guarantees \(\epsilon\)-optimality while allowing the returned constraint value vector to lie within distance \(\epsilon\) of the constraint set, as specified in \citet[Definition~2.4]{miryoosefi2022simple}. For the scalar constraint considered in this work, the returned policy may therefore violate the constraint by up to \(\epsilon\):
\(
V_{c,1}^{P,\pi}(s_1)\geq\ell-\epsilon.
\)
In contrast, Objective~\ref{objective} requires exact feasibility:
\(
V_{c,1}^{P,\pi}(s_1)\geq\ell.
\)

\item Both \citet{qu2025hybrid,wu2026unified} consider finite-horizon MDPs with stationary transition dynamics and state their sample-complexity bounds in terms of the number of online iterations. For a consistent comparison, Table~\ref{tab:comparison} expresses all bounds in terms of the total number of state-action samples collected from the real environment under non-stationary dynamics. We therefore convert the prior bounds in two steps: we multiply the number of online iterations by \(H\) to obtain the number of state-action samples, and we replace \(SA\) with \(HSA\) to account for a separate transition kernel at each timestep.

Moreover, \citet{qu2025hybrid} assume \(\sigma_r\)-reachability: there exists a \textbf{known}
constant \(\sigma_r\in(0,1]\) such that, for every \((h,s,a)\in[H]\times\mathcal S\times\mathcal A\),
\[
\max_{\pi\in\Pi} d_h^{P^{\mathrm{real}},\pi}(s,a)\ge \sigma_r,
\]
where \(d_h^{P^{\mathrm{real}},\pi}(s,a)\) is the probability of visiting \((s,a)\) at timestep \(h\)
under policy \(\pi\). As shown in Table~\ref{tab:comparison}, their sample complexity depends
polynomially on \(1/\sigma_r\). In contrast, our analysis does not require this reachability
assumption and therefore has no dependence on \(\sigma_r\).

\citet{qu2025hybrid} also impose a strict separation condition: there exists a \textbf{known} constant \(\sigma_s\in(0,1]\) such that, for every \((h,s,a)\in[H]\times\mc S\times\mc A\),
\[
\dtv{P_h^{\mathrm{real}}(\cdot\mid s,a)}
    {P_h^{\mathrm{sim}}(\cdot\mid s,a)}
\in \{0\}\cup[\sigma_s,1].
\]
Thus, the simulator transition must either match the real transition exactly or differ from it by at least \(\sigma_s\). Our Assumption~\ref{ass:separation} relaxes this condition by allowing the discrepancy to lie in
\[
[0,\epsilon_s]\cup[\sigma_s,1].
\]
This relaxation allows for a small, nonzero modeling errors and therefore captures a broader class of sim-to-real mismatches.
\end{itemize}
\end{remark}

\section{Experimental details}\label{app:exp_details}
\paragraph{Computational resources.}
All experiments were conducted on a single desktop machine equipped with an AMD Ryzen~9 9950X3D CPU (16 cores) and 128 GB RAM. No GPU was used. All planning subproblems were formulated as linear programs over the occupancy-measure polytope and solved using the HiGHS solver through SciPy. In our largest instance, each linear program contained $H|\mathcal S||\mathcal A|=1200$ decision variables.
\paragraph{Stationary adaptation.}
Here, we describe how we adapt Algorithm~\ref{alg:hybrid} to the stationary setting, where the transition dynamics are time-independent. In this case, visitation counts are aggregated across timesteps:
\[
n^t(s,a)=\sum_{i=1}^{t-1}\sum_{h=1}^H\mathbf{1}_{\{(s_h^i,a_h^i)=(s,a)\}},\;n^t(s,a,s')=\sum_{i=1}^{t-1}\sum_{h=1}^H\mathbf{1}_{\{(s_h^i,a_h^i,s_{h+1}^i)=(s,a,s')\}}.
\]
The empirical model is
\[
\hat P^t(s'\mid s,a)=\mathbf{1}_{\{n^{t}(s,a)>0\}}\frac{n^{t}(s,a,s')}{n^{t}(s,a)}
+
\mathbf{1}_{\{n^{t}(s,a)=0\}}\frac{1}{|\mc S|},
\]
The mismatch region and the exploration bonuses are indexed by \((s,a)\) instead of \((h,s,a)\), and each trajectory contributes \(H\) samples to the aggregated counts. 
\subsection{Baseline algorithms}
In this section, we describe two baseline algorithms: a reward-free safe RL algorithm and an unconstrained sim-to-real RL algorithm. Both are modifications of Algorithm~\ref{alg:hybrid} and are designed to isolate the benefits of simulator access and safe exploration, respectively.
\paragraph{Reward-free safe RL setting.} 
To adapt Algorithm~\ref{alg:hybrid} to the reward-free safe RL setting, we treat all state-action pairs as mismatch triples, i.e., \(\mc B=[H]\times\mc S\times\mc A\). In this case, the algorithm does not exploit simulator information and reduces to an online reward-free safe  RL algorithm. Correspondingly, we define the empirical model \(\hat P^{t}\) and the bonus \(b^t\) as
\begin{equation}\label{eq:new}
\begin{aligned}
\hat P_h^{t}(s'\mid s,a)\coloneqq\hat P_h^{t,\mathrm{real}}(s'\mid s,a)\ \text{as in Equation~\eqref{eq_def_hat_P_real}},\quad
b_h^t(s,a)\coloneqq H\rho_h^t(s,a).
\end{aligned}
\end{equation}
We present reward-free safe RL algorithm below.
\begin{algorithm}[H]
\caption{Reward-free safe RL}
\label{alg:reward-free safe RL}
\begin{algorithmic}[1]
\Require A strictly feasible baseline policy $\pi^0$ with margin $\xi$, safety threshold $\ell$, accuracy $\tau\le \xi/4$,  and access to $\mc M^{\mathrm{real}}$, confidence level $\delta\in(0,1)$
\State Initialize $n_h^0(s,a)\gets0$ for all $(h,s,a)\in[H]\times\mc S\times\mc A$.
\For{$t=1,2,\ldots$}
    \State Construct $\hat P^{t}$ and $b^t$ from trajectories $\{\tau^i\}_{i=1}^{t-1}$, as defined in Equation \eqref{eq:new}. 
    \If{$V_{c,1}^{\hat P^t,\pi^0}(s_1)<\ell+\xi/2$}
        \State $\pi^t\gets\pi^0$
    \Else \Comment{Computing safe exploration policy}
        \State $\bar\pi^t\in\arg\max_{\pi\in\Pi}\left\{V_{b,1}^{\hat P^t,\pi}(s_1):V_{c,1}^{\hat P^t,\pi}(s_1)\geq\ell\right\}$
        \State $\Delta^t\gets\min\left\{H,\,V_{b,1}^{\hat P^t,\bar\pi^t}(s_1)\right\}$
        \If{$\Delta^t\leq\tau/2$}
            \State \textbf{break} the \textbf{for} loop
        \EndIf
        \State $\pi^t\gets\alpha^t\bar\pi^t+(1-\alpha^t)\pi^0$, where $\alpha^t$ is defined in Equation \eqref{eq_def_policy}
    \EndIf
    \State Roll out $\pi^t$ in $\mc M^{\mathrm{real}}$ and observe $\tau^t\coloneqq\{(s_h^t,a_h^t,s_{h+1}^t)\}_{h=1}^H$
\EndFor
\State \Return $\hat P^t$ and $\pi^{\mathrm{out}}\in\arg\max_{\pi\in\Pi}\left\{V_{r,1}^{\hat P^t,\pi}(s_1):V_{c,1}^{\hat P^t,\pi}(s_1)\geq\ell+\tau/2\right\}$ for any $r\in\mc F$
\end{algorithmic}
\end{algorithm}  
\paragraph{Unconstrained sim-to-real RL setting.} To adapt Algorithm~\ref{alg:hybrid} to the unconstrained sim-to-real RL setting, we ignore the safety constraint and compute the exploration policy directly by maximizing the bonus value, without enforcing safe exploration. Correspondingly, we choose
\[
\pi^t\in\arg\max_{\pi\in\Pi} V_{b,1}^{\hat P^t,\pi}(s_1)
\]
as the exploration policy. We present the unconstrained sim-to-real RL algorithm below.
\begin{algorithm}[H]
\caption{Unconstrained sim-to-real RL}
\label{alg:unconstrained sim-to-real RL}
\begin{algorithmic}[1]
\Require Accuracy $\varepsilon$, simulator information $\mc M^{\mathrm{sim}}$, and access to $\mc M^{\mathrm{real}}$, separation parameters $\sigma_s$ and $\epsilon_s$, confidence level $\delta\in(0,1)$, and $(c,\ell)$ if constrained planning is requested
\State Initialize $n_h^0(s,a)\gets0$ for every triple and $\hat{\mathcal B}^0\gets[H]\times\mathcal S\times\mathcal A$.
\For{$t=1,2,\ldots$}
    \State Construct \(\hat{\mc B}^t\) as in Equations \eqref{eq_deletetion_rule} and \eqref{def:estimate_mismatch} \Comment{ Estimating the mismatch region}
    \State Construct \(\hat P^t\) as in Equations \eqref{eq_def_hat_P_real} and \eqref{eq_def_hybrid}  \Comment{Constructing the hybrid model}
        \State $\pi^t\in\arg\max_{\pi\in\Pi}V_{b,1}^{\hat P^t,\pi}(s_1)$ \Comment{Computing exploration policy}
        \State $\Delta^t\gets\min\left\{H,\, V_{b,1}^{\hat P^t,{\pi^t}}(s_1)\right\}$ 
        \If{$\Delta^t\leq\tau/2$}
            \State \textbf{break} the \textbf{for} loop
        \EndIf
    \State Roll out $\pi^t$ in $\mc M^{\mathrm{real}}$ and observe $\tau^t\coloneqq\{(s_h^t,a_h^t,s_{h+1}^t)\}_{h=1}^H$
\EndFor
\State \Return $\hat P^t$ and $\pi^{\mathrm{out}}\in\arg\max_{\pi\in\Pi}\left\{V_{r,1}^{\hat P^t,\pi}(s_1):V_{c,1}^{\hat P^t,\pi}(s_1)\geq\ell+\tau/2\right\}$ for any $r\in\mc F$
\end{algorithmic}
\end{algorithm}

\subsection{Sample complexity versus the size of the mismatch region \texorpdfstring{$|\mc B|$}{|B|}} \label{app_abl_B}
In this section, we study how the size of the mismatch region, \(|\mc B|\), affects the sample complexity of Algorithm~\ref{alg:hybrid} in achieving safe exploration and accurate planning.

The size of the mismatch region is controlled by the number of windy cells in the gridworld. Since each windy cell affects all four actions, \(n\) windy cells correspond to \(|\mc B|=4n\) mismatch state-action pairs. We vary the number of windy cells so that \(|\mc B|\in\{4,8,12,24,40,64,84,100\}\), out of \(|\mc S||\mc A|=100\). The default configuration has three windy cells adjacent to the unsafe wall, yielding \(|\mc B|=12\). Windy cells are selected according to a fixed priority order rather than by uniform sampling. We first rank all regular cells, excluding the three unsafe cells and the initial state, by increasing Manhattan distance to the center of the unsafe wall. Cells at the same distance are ordered randomly using the instance seed. The unsafe cells are appended next, followed by the initial state \(s_1\). The first \(n\) cells in this ordering are then designated as windy cells. This construction has two important properties. First, mismatch is introduced closest to the unsafe wall first, so the sparse-mismatch regime is safety-critical rather than benign. Second, the construction is incremental: the windy set for a smaller \(|\mc B|\) is always contained in the windy set for any larger \(|\mc B|\). Thus, the experiment adds mismatch to a fixed layout instead of redrawing unrelated instances, making the growth in Figure~\ref{fig:abl_B} attributable to the size of \(\mc B\).

As shown in Figure~\ref{fig:abl_B}, the number of iterations required by Algorithm~\ref{alg:hybrid} to achieve safe exploration and accurate planning grows approximately linearly with \(|\mc B|\). When the mismatch region covers the entire state-action space, Algorithm~\ref{alg:hybrid} no longer exploits simulator information, and the total number of iterations coincides with that of the reward-free safe RL baseline.
\begin{figure}[htbp]
    \centering
    \includegraphics[width=0.78\linewidth]{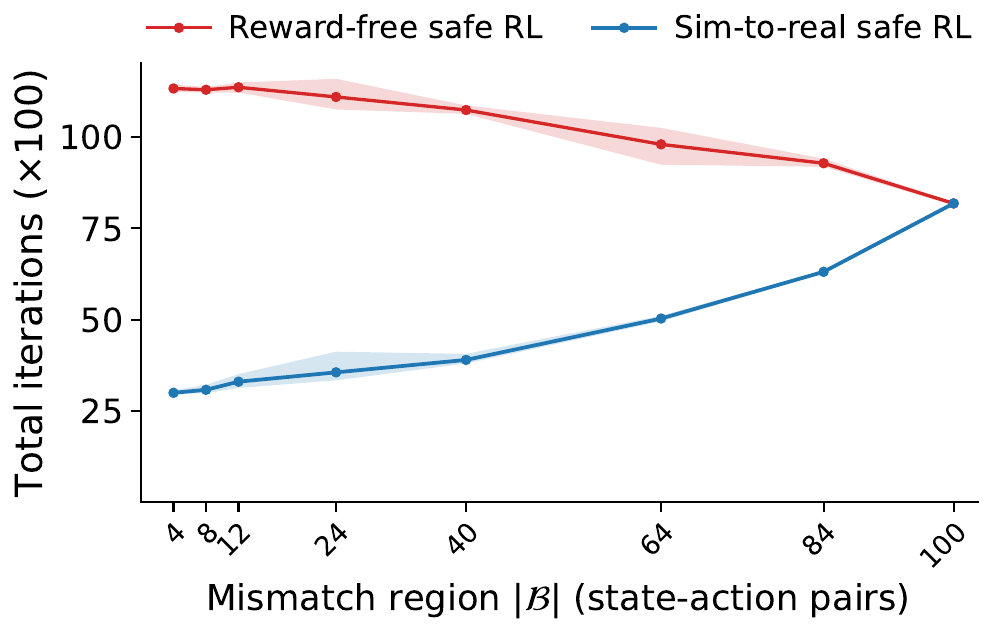}
    \caption{Total iterations versus the size of the mismatch region, \(|\mc B|\), at fixed \(\sigma_s=0.68\) and \(\epsilon_s=0\), where all runs ensure safe exploration and accurate planning. Shaded regions indicate the min--max range across \(10\) seeds.}
    \label{fig:abl_B}
\end{figure}
\subsection{Sample complexity versus the value of the separation parameter \texorpdfstring{$\sigma_s$}{sigma s}} \label{app_exp_sigma_s_1}
In this section, we study how different values of the separation parameter \(\sigma_s\) affect the sample complexity of Algorithm~\ref{alg:hybrid} in achieving safe exploration and accurate planning.

To vary value of \(\sigma_s\), we vary the wind strength \(p_{\mathrm{wind}}\), defined as the probability that the wind reverses the intended move. On a windy cell, the real transition is
\[
P^{\mathrm{real}}(\cdot\mid s,a)
=
(1-p_{\mathrm{wind}})P^{\mathrm{sim}}(\cdot\mid s,a)
+
p_{\mathrm{wind}}P^{\mathrm{sim}}(\cdot\mid s,\bar a),
\]
where \(\bar a\) denotes the action opposite to \(a\). Hence, the total-variation gap from the simulator is
\[
p_{\mathrm{wind}}
\left\|
P^{\mathrm{sim}}(\cdot\mid s,\bar a)
,
P^{\mathrm{sim}}(\cdot\mid s,a)
\right\|_{\mathrm{TV}},
\]
which scales linearly with \(p_{\mathrm{wind}}\). We vary
\[
p_{\mathrm{wind}}\in\{0.15,0.2,0.3,0.4,0.55,0.7,0.85\},
\]
which gives
\[
\sigma_s\in\{0.1275,\,0.17,\,0.255,\,0.34,\,0.4675,\,0.595,\,0.7225\}.
\]

As shown in Figure~\ref{fig:abl_sigma}, Algorithm~\ref{alg:hybrid} requires more iterations when \(\sigma_s\) is smaller, because non-mismatch state-action pairs become harder to distinguish from mismatch state-action pairs. When the separation is sufficiently small, the total number of iterations approaches that of the reward-free safe RL baseline.
\begin{figure}[htbp]
    \centering
    \includegraphics[width=0.78\linewidth]{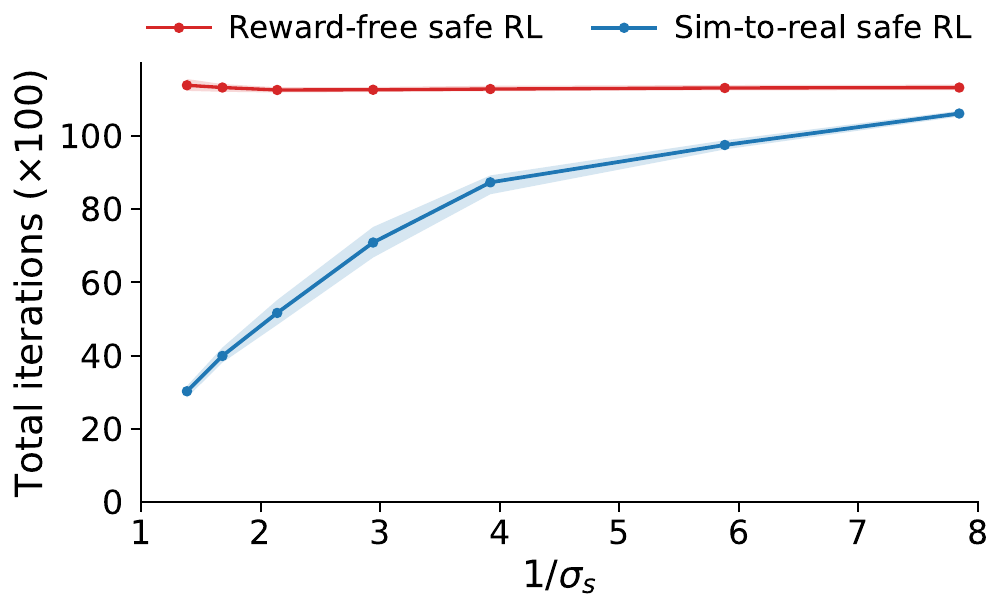}
    \caption{Total iterations versus $1/\sigma_s$ at fixed $|\mc B|=12$ and $\epsilon_s=0$, where all runs ensure safe exploration and accurate planning. Shaded regions indicate the min--max range across $10$ seeds.\looseness-1}
    \label{fig:abl_sigma}
\end{figure}
\subsection{Estimating the separation parameter \texorpdfstring{$\sigma_s$}{sigma-s} from real-world data}\label{app_exp_sigma_s_2}
In our experiments, the simulator matches the real dynamics exactly outside the mismatch region, so \(\epsilon_s=0\) and the separation gap reduces to \(\sigma_s\). We therefore focus on \(\sigma_s\), which determines how easily mismatch triples can be identified. Algorithm~\ref{alg:hybrid} uses \(\sigma_s\) only in the deletion rule \(\mc G^t\) in Line~3, where it sets the threshold \(\sigma_s/2\) for estimating \(\hat{\mc B}^t\). Because \(\sigma_s\) is generally unknown, Appendix~\ref{app_est_sigma_s_alg} presents the estimation procedure in Algorithm~\ref{alg:estimate_sigma_s} and its theoretical guarantees in Proposition~\ref{prop:calibration}. Appendix~\ref{app_sample_sig_s} studies the number of real-world samples required for \(\hat\sigma_s\) to approach \(\sigma_s\). 
\subsubsection{Algorithm for estimating \texorpdfstring{$\sigma_s$}{sigma-s} and its guarantees}\label{app_est_sigma_s_alg}
To estimate \(\sigma_s\), we roll out the strictly feasible baseline policy \(\pi^0\) in \(\mc M^{\mathrm{real}}\) for \(N\) episodes, which ensures safe exploration by Assumption~\ref{ass:slater}. We then use the collected trajectories to estimate the separation parameter \(\sigma_s\) as follows.

Given the trajectories \(\{\tau^i\}_{i=1}^{N}\) collected from \(\mc M^{\mathrm{real}}\) by the policy \(\pi^0\), we define the visitation counts for each \((h,s,a)\in[H]\times\mc S\times\mc A\) as
\[
n_h(s,a)\coloneqq\sum_{i=1}^{N}\mathbf{1}_{\{(s_h^i,a_h^i)=(s,a)\}}, \;\;
n_h(s,a,s')\coloneqq\sum_{i=1}^{N}\mathbf{1}_{\{(s_h^i,a_h^i,s_{h+1}^i)=(s,a,s')\}}.
\]
The empirical transition model estimated from real-world data is defined as
\begin{align}
    \hat P_h^{\mathrm{real}}(s'\mid s,a)\coloneqq\mathbf{1}_{\{n_h(s,a)>0\}}\frac{n_h(s,a,s')}{n_h(s,a)}+\mathbf{1}_{\{n_h(s,a)=0\}}\frac{1}{|\mc S|},\label{eq_empiracl}
\end{align}
and its statistical uncertainty is quantified by the confidence bound
\[{\rho_h(s,a)=\min\left\{1,\sqrt{\beta(n_h(s,a),\delta)/(2n_h(s,a)\vee1)}\right\},}
\]
where \(\beta(n,\delta)\coloneqq\log(2|\mc S||\mc A|H/\delta)+|\mc S|\log(8e(n+1))\).

The quantity we estimate is the minimum \emph{positive} total variation gap over the mismatch region. For every \((h,s,a)\in[H]\times\mc S\times\mc A\), let
\begin{align}
g_h(s,a)\coloneqq\dtv{P_h^{\mathrm{real}}(\cdot\mid s,a)}{P_h^{\mathrm{sim}}(\cdot\mid s,a)},\label{eq_def_gap}
\end{align}
so that mismatch region \(\mc B=\{(h,s,a):g_h(s,a)>0\}\) by Definition~\ref{def:mismatch_region} and \(\sigma_s=\min_{(h,s,a)\in\mc B}g_h(s,a)\) is the constant for which Assumption~\ref{ass:separation} holds. Neither \(\mc B\) nor the gaps \(g_h\) are observable, since \(P^{\mathrm{real}}\) is unknown.

Then, for each \((h,s,a)\in[H]\times\mc S\times\mc A\), we compute a lower confidence bound on \(g_h(s,a)\),
\begin{align}
L_h(s,a)\coloneqq\left(\dtv{\hat P_h^{\mathrm{real}}(\cdot\mid s,a)}{P_h^{\mathrm{sim}}(\cdot\mid s,a)}-\rho_h(s,a)\right)_{\!+}.
\label{eq:gap_lcb}
\end{align}
Consider
\begin{align}
    \mc C\coloneqq\bigl\{(h,s,a):n_h(s,a)\geq n_{\min}\;\text{and}\;\;L_h(s,a)>0\bigr\}, \label{eq_def_set_c}
\end{align}
the set of triples that are visited at least \(n_{\min}\) times and whose gap is provably positive; as we show below, every such triple is certified to be a mismatch triple. Unvisited triples have \(\rho_h(s,a)=1\) and hence \(L_h(s,a)=0\), so they never enter \(\mc C\). Finally, we compute the estimated separation parameter as
\begin{align}
\hat\sigma_s\coloneqq\min\Bigl\{1,\;\min_{(h,s,a)\in\mc C}L_h(s,a)\Bigr\},
\qquad\text{and}\qquad \hat\sigma_s\coloneqq0 \;\text{ if }\; \mc C=\emptyset.
\label{eq:sigma_est}
\end{align}
The procedure of estimation of $\sigma_s$ is summarized in Algorithm~\ref{alg:estimate_sigma_s}. 

\begin{algorithm}[H]
\caption{Safe estimation of the separation parameter $\sigma_s$}
\label{alg:estimate_sigma_s}
\begin{algorithmic}[1]
\Require Strictly feasible baseline policy $\pi^0$, simulator information $\mc M^{\mathrm{sim}}$, number of iterations $N$ from $\mc M^{\mathrm{real}}$, confidence level $\delta$, coverage requirement $n_{\min}\ge1$
\For{$i=1,\ldots,N$}
    \State Roll out $\pi^0$ in $\mc M^{\mathrm{real}}$ and observe $\tau^i\coloneqq\{(s_h^i,a_h^i,s_{h+1}^i)\}_{h=1}^H$ 
\EndFor
\State Compute $\hat P^{\mathrm{real}}$ as in Equation \eqref{eq_empiracl}, $L_h$ as in Equation \eqref{eq:gap_lcb} and $\mc C$ as in Equation \eqref{eq_def_set_c}.
\If{$\mc C\neq\emptyset$}
\State \Return $\hat\sigma_s\gets\min\{1,\min_{(h,s,a)\in\mc C}L_h(s,a)\}$
\Else
\State \Return $\hat\sigma_s\gets0$
\EndIf
\end{algorithmic}
\end{algorithm}
In the following, we characterize the statistical guarantee for \(\hat\sigma_s\) returned by Algorithm~\ref{alg:estimate_sigma_s}.

\begin{proposition}\label{prop:calibration}
Let Assumption~\ref{ass:separation} hold with \(\epsilon_s=0\), and let \(\sigma_s\) denote the minimum positive gap over the mismatch region \(\mc B\). Then, with probability at least \(1-\delta\), Algorithm~\ref{alg:estimate_sigma_s} guarantees the following:
\begin{enumerate}
\item[(i)] \(\mc C\subseteq\mc B\).
\item[(ii)]
\(
\hat\sigma_s\leq\min_{(h,s,a)\in\mc C}\dtv{P_h^{\mathrm{real}}(\cdot\mid s,a)}{P_h^{\mathrm{sim}}(\cdot\mid s,a)}=\min_{(h,s,a)\in\mc C}g_h(s,a),
\)
where the minimum over an empty set is interpreted as \(+\infty\).
\item[(iii)] If \(\emptyset\neq\mc B\subseteq\mc C\), then
\[
\sigma_s-2\max_{(h,s,a)\in\mc B}\rho_h(s,a)\leq\hat\sigma_s\leq\sigma_s.
\]
\end{enumerate}
\end{proposition}
Proposition~\ref{prop:calibration} shows that Algorithm~\ref{alg:estimate_sigma_s} provides a conservative estimate of both the mismatch region and the separation parameter: with high probability, every triple in \(\mc C\) is a true mismatch triple, and \(\hat\sigma_s\) lower bounds the smallest simulator-real gap over this set. Consequently, if the collected data are sufficiently rich to identify all mismatch triples, i.e., \(\mc B\subseteq\mc C\), then \(\hat\sigma_s\le\sigma_s\).

As shown in Appendix~\ref{app_sample_sig_s}, when sufficient data are sampled from \(\mc M^{\mathrm{real}}\), the estimated separation parameter \(\hat\sigma_s\) approaches the true value \(\sigma_s\) with high probability. In this case, \(\mc C=\mc B\), and \(\hat\sigma_s\) is a conservative estimate of \(\sigma_s\). Furthermore, as shown in Appendix~\ref{app_exp_sigma_s_3} and implied by Theorem~\ref{thm:main}, such a conservative estimate is sufficient for Algorithm~\ref{alg:hybrid} to achieve safe exploration and accurate planning, at the cost of increased sample complexity.

Next, we present proof of Proposition~\ref{prop:calibration}. 
\begin{proof}[Proof of Proposition~\ref{prop:calibration}]
By Lemma~\ref{lemma:max_ineq_categorical}, applied to each \((h,s,a)\) at level \(\delta/(2|\mathcal S||\mathcal A|H)\), and by a union bound over the \(H|\mathcal S||\mathcal A|\) triples, exactly as in Lemma~\ref{lem:proba_master_event}, with probability at least
\(1-\delta\),
\begin{align}
    \KL\!\left(\hat P_h^{\mathrm{real}}(\cdot\mid s,a),P_h^{\mathrm{real}}(\cdot\mid s,a)\right)\leq\frac{\beta(n_h(s,a),\delta)}{n_h(s,a)}, \, \forall (h,s,a)\in[H]\times\mathcal S\times\mathcal A.
\end{align}
Condition on this event. For every triple with \(n_h(s,a)\geq1\), Pinsker's inequality gives
\[
\dtv{\hat P_h^{\mathrm{real}}(\cdot\mid s,a)}{P_h^{\mathrm{real}}(\cdot\mid s,a)}\leq\rho_h(s,a).
\]
By the triangle inequality,
\[
\left|\dtv{\hat P_h^{\mathrm{real}}(\cdot\mid s,a)}{P_h^{\mathrm{sim}}(\cdot\mid s,a)}-g_h(s,a)\right|\leq\rho_h(s,a).
\]
Consequently,
\[
L_h(s,a)=\left(\dtv{\hat P_h^{\mathrm{real}}(\cdot\mid s,a)}{P_h^{\mathrm{sim}}(\cdot\mid s,a)}-\rho_h(s,a)\right)_{+}\leq g_h(s,a).
\]
Thus, \(L_h(s,a)>0\) implies \(g_h(s,a)>0\), so every triple in \(\mathcal C\) belongs to \(\mathcal B\). This proves (i).

For (ii), if \(\mathcal C=\emptyset\), Algorithm~\ref{alg:estimate_sigma_s} returns \(\hat\sigma_s=0\). If \(\mathcal C\neq\emptyset\), then
\[
\hat\sigma_s=\min\left\{1,\min_{(h,s,a)\in\mathcal C}L_h(s,a)\right\}\leq\min_{(h,s,a)\in\mathcal C}L_h(s,a)\leq\min_{(h,s,a)\in\mathcal C}g_h(s,a).
\]

For (iii), we assume that \(\mc B\neq\emptyset\). Indeed, if \(\mc B=\emptyset\), then the simulator is exact, \(\sigma_s\) is unconstrained by Assumption~\ref{ass:separation}, and part~(i) forces \(\mc C=\emptyset\) and \(\hat\sigma_s=0\), so the claim is vacuous. Now suppose that \(\mc B\subseteq\mc C\) and \(\mc B\neq\emptyset\). Since \(\mathcal B\neq\emptyset\), this implies \(\mathcal C\neq\emptyset\). Together with (i), it gives \(\mathcal C=\mathcal B\). Hence,
  \[
  \hat\sigma_s\leq\min_{(h,s,a)\in\mathcal C}g_h(s,a)=\min_{(h,s,a)\in\mathcal B}g_h(s,a)=\sigma_s.
  \]
  Furthermore,
  \[
  \dtv{\hat P_h^{\mathrm{real}}(\cdot\mid s,a)}{P_h^{\mathrm{sim}}(\cdot\mid s,a)}\geq g_h(s,a)-\rho_h(s,a),
  \]
  so
  \[
  L_h(s,a)\geq\left(g_h(s,a)-2\rho_h(s,a)\right)_{+}\geq g_h(s,a)-2\rho_h(s,a).
  \]
  Using \(\mathcal C=\mathcal B\),
  \[
  \begin{aligned}
  \min_{(h,s,a)\in\mathcal C}L_h(s,a)&\geq\min_{(h,s,a)\in\mathcal B}\left(g_h(s,a)-2\rho_h(s,a)\right)\\
  &\geq\min_{(h,s,a)\in\mathcal B}g_h(s,a)-2\max_{(h,s,a)\in\mathcal B}\rho_h(s,a)\\
  &=\sigma_s-2\max_{(h,s,a)\in\mathcal B}\rho_h(s,a).
  \end{aligned}
  \]
  Since \(\sigma_s-2\max_{(h,s,a)\in\mathcal B}\rho_h(s,a)\leq\sigma_s\leq1\), clipping at \(1\) preserves this lower bound. Therefore,
  \[
  \sigma_s-2\max_{(h,s,a)\in\mathcal B}\rho_h(s,a)\leq\hat\sigma_s\leq\sigma_s.
  \]
  \end{proof}
\subsubsection{Sample complexity of estimating \texorpdfstring{$\sigma_s$}{sigma-s}}
\label{app_sample_sig_s}
In this section, we investigate the sample complexity required to obtain a sufficiently accurate estimate \(\hat\sigma_s\). Here, sample complexity refers to the number of samples collected by rolling out the safe baseline policy \(\pi^0\) in \(\mc M^{\mathrm{real}}\).

Using the grid layout in Figure~\ref{rf:fig:gridworld}, we vary the \emph{profile} of the wind strength across the mismatch region. We consider three instances: (i) a homogeneous instance with \(p_{\mathrm{wind}}=0.8\); (ii) a homogeneous instance with \(p_{\mathrm{wind}}=0.35\), which has a smaller \(\sigma_s\) and therefore requires more samples to estimate accurately; and (iii) a heterogeneous instance in which \(p_{\mathrm{wind}}\) takes equispaced values from \(0.2\) to \(0.8\). In the heterogeneous instance, windy cells are ordered by increasing distance from the unsafe wall, with wind strength increasing along this order. Consequently, the weakest and hardest-to-certify mismatch lies closest to the unsafe wall, making this instance more challenging. We summarize the three instances in the table below.
\begin{center}
\begin{tabular}{lcc}
\toprule
Instance & $p_{\mathrm{wind}}$  & $\sigma_s$ \\
\midrule
I (homogeneous) & $0.8$ everywhere & $0.68$ \\
II (homogeneous) & $0.35$ everywhere & $0.2975$ \\
III (heterogeneous) & $0.2$ to $0.8$, equispaced & $0.17$ \\
\bottomrule
\end{tabular}
\captionof{table}{Three gridworld instances}\label{tabel_instance}
\end{center}
As shown in Figure~\ref{fig:abl_sigma_s_estimation}, \(\hat\sigma_s\) approaches \(\sigma_s\) from below as the number of samples increases, providing a conservative estimate. The estimated mismatch set \(\mc C\) coincides with the true mismatch region \(\mc B\) after \(2.4\times10^5\), \(1.2\times10^6\), and \(6\times10^5\) samples for Instances I, II, and III, respectively. Instance I is the easiest to estimate because it has the largest \(\sigma_s\), whereas Instances II and III require more samples because their separation parameters are smaller.
\begin{figure}[htbp]
    \centering
    \includegraphics[width=\linewidth]{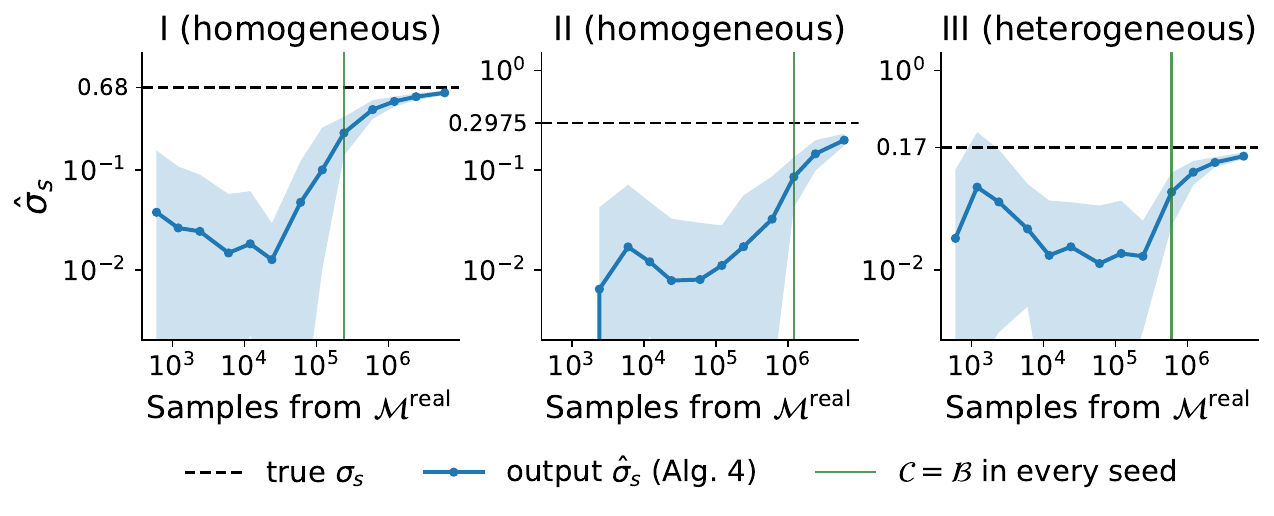}
    \caption{Sample complexity of estimating $\sigma_s$ in the three instances. The curves show the mean over \(20\) seeds, with shaded regions indicating the min--max range.}
    \label{fig:abl_sigma_s_estimation}
\end{figure}
\subsection{Effect of misspecifying \texorpdfstring{$\sigma_s+\epsilon_s$}{sigma-s}}
\label{app_exp_sigma_s_3}
In this section, we investigate the effect of replacing \(\sigma_s+\epsilon_s\) with its estimate \(\hat\sigma_s+\hat\epsilon_s\) in Algorithm~\ref{alg:hybrid}. This quantity determines the threshold used in Line~3 to estimate the mismatch region \({\mc B}\).

We consider the same three instances as in Table~\ref{tabel_instance}. Since \(\epsilon_s=0\) in each instance, we set \(\hat\epsilon_s=0\) and vary \(\hat\sigma_s\) from \(0\) to \(1\), covering both underestimation and overestimation. For each instance, we test \(20\) values and evaluate sample complexity, safe exploration, and planning accuracy. Planning accuracy requires the returned policy to be both near-optimal and feasible in \(\mc M^{\mathrm{real}}\).

Table~\ref{fig:abl_sigma_s_estimation_3} reports the results. No deployed policy violated the safety constraint for any tested value of \(\hat\sigma_s+\hat\epsilon_s\). Underestimating \(\sigma_s+\epsilon_s\) preserves planning accuracy but makes the algorithm more conservative, thereby increasing the required real-world samples. In the extreme case \(\hat\sigma_s+\hat\epsilon_s=0\), the algorithm reduces to online reward-free safe RL. Overestimation has the opposite effect: it can reduce sample complexity but may cause true mismatch triples to be removed from \(\hat{\mc B}^t\), making the resulting model inaccurate for subsequent constrained reward-free planning. False deletions first occur at \(\hat\sigma_s+\hat\epsilon_s\approx2.5(\sigma_s+\epsilon_s)\) for Instance~II and at \(\hat\sigma_s+\hat\epsilon_s\approx3.2(\sigma_s+\epsilon_s)\) for Instance~III.

Overall, conservative underestimation preserves the planning-accuracy and safe-exploration guarantees of Theorem~\ref{thm:main}, but at the cost of higher sample complexity. By contrast, overestimation may remove true mismatch triples from \(\hat{\mc B}^t\), compromising planning accuracy as observed experimentally, and potentially safe exploration. Although no unsafe exploration was observed in our experiments, the theoretical safety guarantee need not hold after such false removals. Therefore, a conservative estimate of \(\sigma_s+\epsilon_s\) is preferable.

\begin{table}[htbp]
\centering
\caption{Performance of Algorithm~\ref{alg:hybrid} with different values of \(\hat\sigma_s+\hat \epsilon_s\). Samples from \(\mc M^{\mathrm{real}}\): total number of samples required from \(\mc M^{\mathrm{real}}\) before termination, reported as mean \(\pm\) standard deviation over 20 seeds. Unsafe iterations: share of seeds for which there exists at least one iteration \(t\) such that the deployed policy \(\pi^t\) is infeasible during the execution of Algorithm~\ref{alg:hybrid}. Pairs of \(\mc B\) lost: mean number of mismatch pairs from \(\mc B\) incorrectly removed from \(\hat{\mc B}^t\), averaged over 20 seeds. Runs with false deletion: share of seeds in which at least one pair from \(\mc B\) is lost. Output optimality: \(V_r^{\pi^*}-V_r^{\pi}\), averaged over \(10\) random reward functions per seed, where \(\pi\) is the output optimal policy for reward \(r\), and \(\pi^*\) is the true optimal policy for reward \(r\). Output feasibility: share of output policies that are feasible in \(\mc M^{\mathrm{real}}\). Rows in \textbf{bold} indicate settings in which the output policy is infeasible in \(\mc M^{\mathrm{real}}\) for at least one reward and seed.}
\label{fig:abl_sigma_s_estimation_3}
\scriptsize
\setlength{\tabcolsep}{4.5pt}
\renewcommand{\arraystretch}{0.92}
\providecommand{\bcell}[1]{\textbf{\boldmath #1}}
\begin{tabular}{lcccccc}
\toprule
$\hat\sigma_s+\hat \epsilon_s$ & Samples & Unsafe iterations & Pairs of $\mc B$ lost & Runs with false & Optimality & Output \\
 & & (share of seeds) & (out of $64$) & deletion & $V^{*}-V^{\pi}$ & feasible \\
\midrule
\multicolumn{7}{l}{\small\textbf{I (homogeneous)}, $\sigma_s=0.6800$ and $\epsilon_s=0$} \\[2pt]
0.000 & $115{,}878\pm3{,}135$ & $0.0\%$ & $0.0$ & $0\%$ & $0.0041$ & $100.0\%$ \\
0.050 & $115{,}878\pm3{,}135$ & $0.0\%$ & $0.0$ & $0\%$ & $0.0041$ & $100.0\%$ \\
0.100 & $112{,}266\pm3{,}090$ & $0.0\%$ & $0.0$ & $0\%$ & $0.0043$ & $100.0\%$ \\
0.150 & $105{,}744\pm1{,}884$ & $0.0\%$ & $0.0$ & $0\%$ & $0.0046$ & $100.0\%$ \\
0.200 & $104{,}238\pm1{,}232$ & $0.0\%$ & $0.0$ & $0\%$ & $0.0044$ & $100.0\%$ \\
0.250 & $98{,}163\pm770$ & $0.0\%$ & $0.0$ & $0\%$ & $0.0044$ & $100.0\%$ \\
0.300 & $91{,}086\pm1{,}227$ & $0.0\%$ & $0.0$ & $0\%$ & $0.0042$ & $100.0\%$ \\
0.340 & $85{,}677\pm1{,}336$ & $0.0\%$ & $0.0$ & $0\%$ & $0.0041$ & $100.0\%$ \\
0.400 & $77{,}085\pm1{,}101$ & $0.0\%$ & $0.0$ & $0\%$ & $0.0046$ & $100.0\%$ \\
0.450 & $71{,}586\pm852$ & $0.0\%$ & $0.0$ & $0\%$ & $0.0044$ & $100.0\%$ \\
0.500 & $66{,}771\pm730$ & $0.0\%$ & $0.0$ & $0\%$ & $0.0044$ & $100.0\%$ \\
0.550 & $63{,}795\pm536$ & $0.0\%$ & $0.0$ & $0\%$ & $0.0043$ & $100.0\%$ \\
0.595 & $62{,}361\pm465$ & $0.0\%$ & $0.0$ & $0\%$ & $0.0043$ & $100.0\%$ \\
0.620 & $61{,}356\pm501$ & $0.0\%$ & $0.0$ & $0\%$ & $0.0045$ & $100.0\%$ \\
0.650 & $60{,}636\pm592$ & $0.0\%$ & $0.0$ & $0\%$ & $0.0050$ & $100.0\%$ \\
0.700 & $59{,}571\pm463$ & $0.0\%$ & $0.0$ & $0\%$ & $0.0045$ & $100.0\%$ \\
0.750 & $58{,}701\pm419$ & $0.0\%$ & $0.0$ & $0\%$ & $0.0044$ & $100.0\%$ \\
0.800 & $58{,}074\pm333$ & $0.0\%$ & $0.0$ & $0\%$ & $0.0043$ & $100.0\%$ \\
0.850 & $57{,}411\pm355$ & $0.0\%$ & $0.0$ & $0\%$ & $0.0043$ & $100.0\%$ \\
1.000 & $56{,}229\pm304$ & $0.0\%$ & $0.0$ & $0\%$ & $0.0047$ & $100.0\%$ \\
\midrule
\multicolumn{7}{l}{\small\textbf{II (homogeneous)}, $\sigma_s=0.2975$ and $\epsilon_s=0$} \\[2pt]
0.000 & $118{,}506\pm3{,}659$ & $0.0\%$ & $0.0$ & $0\%$ & $0.0059$ & $100.0\%$ \\
0.050 & $118{,}506\pm3{,}659$ & $0.0\%$ & $0.0$ & $0\%$ & $0.0059$ & $100.0\%$ \\
0.100 & $112{,}647\pm5{,}057$ & $0.0\%$ & $0.0$ & $0\%$ & $0.0059$ & $100.0\%$ \\
0.150 & $105{,}957\pm1{,}332$ & $0.0\%$ & $0.0$ & $0\%$ & $0.0058$ & $100.0\%$ \\
0.200 & $104{,}793\pm1{,}157$ & $0.0\%$ & $0.0$ & $0\%$ & $0.0061$ & $100.0\%$ \\
0.250 & $96{,}492\pm1{,}454$ & $0.0\%$ & $0.0$ & $0\%$ & $0.0062$ & $100.0\%$ \\
0.300 & $88{,}455\pm1{,}358$ & $0.0\%$ & $0.0$ & $0\%$ & $0.0063$ & $100.0\%$ \\
0.340 & $84{,}690\pm1{,}602$ & $0.0\%$ & $0.0$ & $0\%$ & $0.0062$ & $100.0\%$ \\
0.400 & $76{,}836\pm2{,}022$ & $0.0\%$ & $0.0$ & $0\%$ & $0.0060$ & $100.0\%$ \\
0.450 & $71{,}817\pm1{,}605$ & $0.0\%$ & $0.0$ & $0\%$ & $0.0062$ & $100.0\%$ \\
0.500 & $68{,}349\pm1{,}532$ & $0.0\%$ & $0.0$ & $0\%$ & $0.0060$ & $100.0\%$ \\
0.550 & $65{,}568\pm1{,}421$ & $0.0\%$ & $0.0$ & $0\%$ & $0.0061$ & $100.0\%$ \\
0.595 & $63{,}153\pm1{,}014$ & $0.0\%$ & $0.0$ & $0\%$ & $0.0062$ & $100.0\%$ \\
0.620 & $62{,}028\pm1{,}085$ & $0.0\%$ & $0.0$ & $0\%$ & $0.0064$ & $100.0\%$ \\
0.650 & $60{,}555\pm973$ & $0.0\%$ & $0.0$ & $0\%$ & $0.0066$ & $100.0\%$ \\
0.700 & $58{,}773\pm856$ & $0.0\%$ & $0.0$ & $0\%$ & $0.0063$ & $100.0\%$ \\
0.750 & $58{,}338\pm1{,}170$ & $0.0\%$ & $0.9$ & $65\%$ & $0.0145$ & $100.0\%$ \\
0.800 & $57{,}732\pm746$ & $0.0\%$ & $3.8$ & $100\%$ & $0.0451$ & $100.0\%$ \\
\bcell{0.850} & \bcell{$54{,}774\pm1{,}428$} & \bcell{$0.0\%$} & \bcell{$11.6$} & \bcell{$100\%$} & \bcell{$0.0855$} & \bcell{$99.0\%$} \\
\bcell{1.000} & \bcell{$42{,}066\pm1{,}404$} & \bcell{$0.0\%$} & \bcell{$31.7$} & \bcell{$100\%$} & \bcell{$0.1079$} & \bcell{$90.0\%$} \\
\midrule
\multicolumn{7}{l}{\small\textbf{III (heterogeneous)}, $\sigma_s=0.1700$ and $\epsilon_s=0$} \\[2pt]
0.000 & $115{,}719\pm4{,}377$ & $0.0\%$ & $0.0$ & $0\%$ & $0.0068$ & $100.0\%$ \\
0.050 & $115{,}719\pm4{,}377$ & $0.0\%$ & $0.0$ & $0\%$ & $0.0068$ & $100.0\%$ \\
0.100 & $111{,}855\pm4{,}936$ & $0.0\%$ & $0.0$ & $0\%$ & $0.0068$ & $100.0\%$ \\
0.150 & $103{,}686\pm2{,}597$ & $0.0\%$ & $0.0$ & $0\%$ & $0.0068$ & $100.0\%$ \\
0.200 & $102{,}645\pm2{,}264$ & $0.0\%$ & $0.0$ & $0\%$ & $0.0066$ & $100.0\%$ \\
0.250 & $94{,}635\pm2{,}234$ & $0.0\%$ & $0.0$ & $0\%$ & $0.0072$ & $100.0\%$ \\
0.300 & $87{,}321\pm1{,}534$ & $0.0\%$ & $0.0$ & $0\%$ & $0.0070$ & $100.0\%$ \\
0.340 & $83{,}319\pm1{,}626$ & $0.0\%$ & $0.0$ & $0\%$ & $0.0069$ & $100.0\%$ \\
0.400 & $76{,}380\pm1{,}854$ & $0.0\%$ & $0.0$ & $0\%$ & $0.0075$ & $100.0\%$ \\
0.450 & $71{,}778\pm1{,}763$ & $0.0\%$ & $0.0$ & $0\%$ & $0.0067$ & $100.0\%$ \\
0.500 & $67{,}782\pm1{,}751$ & $0.0\%$ & $0.0$ & $0\%$ & $0.0071$ & $100.0\%$ \\
0.550 & $64{,}563\pm1{,}336$ & $0.0\%$ & $0.2$ & $20\%$ & $0.0074$ & $100.0\%$ \\
0.595 & $61{,}755\pm1{,}168$ & $0.0\%$ & $0.6$ & $45\%$ & $0.0069$ & $100.0\%$ \\
0.620 & $60{,}369\pm1{,}155$ & $0.0\%$ & $1.4$ & $75\%$ & $0.0086$ & $100.0\%$ \\
0.650 & $58{,}719\pm1{,}294$ & $0.0\%$ & $2.3$ & $100\%$ & $0.0080$ & $100.0\%$ \\
0.700 & $56{,}523\pm1{,}291$ & $0.0\%$ & $4.4$ & $100\%$ & $0.0097$ & $100.0\%$ \\
0.750 & $53{,}739\pm1{,}303$ & $0.0\%$ & $7.2$ & $100\%$ & $0.0129$ & $100.0\%$ \\
0.800 & $51{,}099\pm1{,}431$ & $0.0\%$ & $9.4$ & $100\%$ & $0.0160$ & $100.0\%$ \\
\bcell{0.850} & \bcell{$48{,}651\pm1{,}549$} & \bcell{$0.0\%$} & \bcell{$11.9$} & \bcell{$100\%$} & \bcell{$0.0238$} & \bcell{$99.0\%$} \\
\bcell{1.000} & \bcell{$43{,}410\pm1{,}333$} & \bcell{$0.0\%$} & \bcell{$17.5$} & \bcell{$100\%$} & \bcell{$0.0372$} & \bcell{$95.5\%$} \\
\bottomrule
\end{tabular}

\end{table}

\section{Proofs}
\subsection{Concentration events}
\label{app:concentration}

Let $d_h^{P^{\mathrm{real}},\pi}(s,a)$ denote the probability of visiting state-action pair $(s,a)$ at timestep $h$ under policy $\pi$ and transition kernel $P^{\mathrm{real}}$. Let $\mathcal F_{t-1}$ denote the $\sigma$-field generated by all observations before iteration $t$, excluding the random draw of the mixture component used in iteration $t$. For every $(h,s,a)\in[H]\times\mc S\times\mc A$ and $t\in\mathbb N$, define the cumulative visitation probability
\[
    \bar{n}_h^t(s,a):= \sum_{i=1}^{t-1} d_h^{P^{\mathrm{real}},\pi^i}(s,a)
\]
This quantity represents the cumulative expected number of visits to $(s,a)$ at timestep $h$ over the first $t-1$ episodes of Algorithm~\ref{alg:hybrid}. When $t=1$, $\bar n_h^1(s,a)=0$. Since $\pi^t$ is $\mathcal F_{t-1}$-measurable, $d_h^{P^{\mathrm{real}},\pi^t}(s,a)$ is the conditional probability, given $\mathcal F_{t-1}$, that $(s,a)$ is visited at time step $h$ in iteration $t$.

We introduce two favorable events. The event $\mathcal E$ ensures that the empirical transition model is close to the true transition model, while $\mathcal E^{\mathrm{cnt}}$ guarantees that the observed visitation counts have a uniform lower bound in terms of their cumulative conditional visitation probabilities:
\begin{small}
\begin{align*}
    &\mathcal{E} := \left\{\forall t \in \mathbb{N},\ \forall h \in [H],\
        \forall (s,a) \in \mc{S}\times\mc{A}:\ \KL\!\left(\hat{P}^{t,\mathrm{real}}_h(\cdot
        \mid s,a),\, P^{\mathrm{real}}_h(\cdot \mid s,a)\right) \leq
        \frac{\beta(n_h^t(s,a),\delta)}{n_h^t(s,a)}\right\},\\
    &\mathcal{E}^{\cnt} := \left\{\forall t \in \mathbb{N},\ \forall h \in
        [H],\ \forall (s,a) \in \mc{S}\times\mc{A}:\ n_h^t(s,a) \geq
        \tfrac{1}{2}\bar{n}_h^t(s,a) - \beta^{\cnt}(\delta)\right\}.
\end{align*}
\end{small}

We next show that, for appropriate choices of the confidence functions $\beta$ and $\beta^{\mathrm{cnt}}$, the events $\mathcal E$ and $\mathcal E^{\mathrm{cnt}}$ hold simultaneously with high probability.

\begin{lemma}
\label{lem:proba_master_event}
By setting
\begin{align*}
    \beta(n,\delta) := \log(2|\mc S||\mc A|H/\delta) + {|\mc S|}\log\!\left(8e(n+1)\right),\;\beta^{\cnt}(\delta) := \log\!\left(2|\mc S||\mc A|H/\delta\right),
\end{align*}
we have $\Pr(\mc{E} \cap \mc{E}^{\cnt}) \geq 1-\delta$.
\end{lemma}
\begin{proof}[Proof of Lemma~\ref{lem:proba_master_event}]
By Lemma~\ref{lemma:max_ineq_categorical}, we have $\Pr(\mc{E}) \geq 1 - \delta/2$.
Similarly, by Lemma~\ref{lemma:bernoulli-deviation}, $\Pr(\mc{E}^{\cnt}) \geq 1 - \delta/2$.
Applying the union bound yields $\Pr(\mc{E} \cap \mc{E}^{\cnt}) \geq 1 - \delta$.
\end{proof}
\begin{lemma}\label{lem:radius-pseudocount}
On the event \(\mathcal E^{\mathrm{cnt}}\), for every \((h,s,a)\in[H]\times\mc S\times\mc A\) and every \(t\in\mathbb N\),
\[
\rho_h^t(s,a)\leq\min\left\{1,\sqrt{\frac{4\beta(\bar n_h^t(s,a),\delta)}{\bar n_h^t(s,a)\vee1}}\right\}.
\]
\end{lemma}

\begin{proof}[Proof of Lemma~\ref{lem:radius-pseudocount}]
If
\(
\bar n_h^t(s,a)\leq4\beta\bigl(\bar n_h^t(s,a),\delta\bigr),
\)
then, using \(\beta\bigl(\bar n_h^t(s,a),\delta\bigr)\geq1\) by (i) from Lemma~\ref{lem:beta-monotone}(i), we have
\[
\frac{4\beta\bigl(\bar n_h^t(s,a),\delta\bigr)}{\bar n_h^t(s,a)\vee1}\geq1.
\]
Thus, the right-hand side of the claimed inequality equals \(1\), and the result follows from \(\rho_h^t(s,a)\leq1\).

Now suppose that
\(
\bar n_h^t(s,a)>4\beta\bigl(\bar n_h^t(s,a),\delta\bigr).
\)
Since \(\beta\bigl(\bar n_h^t(s,a),\delta\bigr)\geq1\), we have \(\bar n_h^t(s,a)>4\), and therefore
\[
\left(\min\left\{1,\sqrt{\frac{4\beta(\bar n_h^t(s,a),\delta)}{\bar n_h^t(s,a)\vee1}}\right\}\right)^2=\frac{4\beta(\bar n_h^t(s,a),\delta)}{\bar n_h^t(s,a)}.
\]
Moreover, since \(\beta\bigl(\bar n_h^t(s,a),\delta\bigr)\geq\beta^{\cnt}(\delta)\), we have
\[
\frac{\bar n_h^t(s,a)}{4}>\beta^{\cnt}(\delta).
\]
On the event \(\mathcal E^{\mathrm{cnt}}\), it follows that
\[
n_h^t(s,a)\geq\frac{\bar n_h^t(s,a)}{2}-\beta^{\cnt}(\delta)>\frac{\bar n_h^t(s,a)}{2}-\frac{\bar n_h^t(s,a)}{4}=\frac{\bar n_h^t(s,a)}{4}.
\]
Since \(z\mapsto\beta(z,\delta)/z\) is nonincreasing by (ii) from Lemma~\ref{lem:beta-monotone} and \(n_h^t(s,a)\geq\bar n_h^t(s,a)/4\),
\[
\bigl(\rho_h^t(s,a)\bigr)^2\leq\frac{\beta(n_h^t(s,a),\delta)}{2n_h^t(s,a)}\leq\frac{2\beta(\bar n_h^t(s,a)/4,\delta)}{\bar n_h^t(s,a)}\leq\frac{2\beta(\bar n_h^t(s,a),\delta)}{\bar n_h^t(s,a)}\leq\frac{4\beta(\bar n_h^t(s,a),\delta)}{\bar n_h^t(s,a)},
\]
where the penultimate inequality uses the fact that \(\beta(\cdot,\delta)\) is nondecreasing. Taking square roots of the above inequality completes the proof.
\end{proof}

Next, we establish the guarantees of Algorithms~\ref{alg:hybrid}, under the events $\mathcal{E}$ and $\mathcal{E}^{\cnt}$, which jointly hold with high probability by
Lemma~\ref{lem:proba_master_event}.
\subsection{Analysis of the estimated mismatch region}
\label{app:cover}
This section proves the theoretical guarantees for Algorithm~\ref{alg:hybrid} on the estimated mismatch region.  Lemma~\ref{lem:cover} shows that, with high probability, no mismatch triple is removed from \(\hat{\mc B}^t\) during learning, while Lemma~\ref{lem:completeness} shows that every non-mismatch triple is removed after being sampled sufficiently often. Finally, Lemma~\ref{lem:identification-budget} bounds the cumulative expected visitation to each non-mismatch triple before it is removed from \(\hat{\mc B}^t\), thereby controlling the total cost of identifying non-mismatch triples.
We prove that the midpoint test retains all triples in \(\mathcal B\) and
eventually deletes each sufficiently sampled calibrated triple.
\begin{lemma}\label{lem:cover}
On the event $\mc E$, for every $t \in \mathbb{N}$, the following hold: (i) $\hat{\mc B}^t \subseteq \hat{\mc B}^{t-1}$; and (ii) $\mc B \subseteq \hat{\mc B}^t$.
\end{lemma}
\begin{proof}[Proof of Lemma~\ref{lem:cover}]
    By definition, $\hat{\mc B}^t = \hat{\mc B}^{t-1} \setminus \mc G^t$, which immediately implies $\hat{\mc B}^t \subseteq \hat{\mc B}^{t-1}$.

To establish (ii), fix $(h,s,a) \in \mc B$ and suppose, towards a contradiction, that $(h,s,a) \in \mc G^t$ for some iteration $t$. By the definition of $\mc G^t$, we have
\[
\dtv{\hat P_h^{t,\mathrm{real}}(\cdot \mid s,a)}{P_h^{\mathrm{sim}}(\cdot \mid s,a)} + \rho_h^t(s,a) \leq \frac{\sigma_s+\epsilon_s}{2}.
\]
Moreover, on the event $\mc E$, Pinsker's inequality yields
\[
\begin{gathered}
\begin{aligned}
\dtv{\hat P_h^{t,\mathrm{real}}(\cdot\mid s,a)}{P_h^{\mathrm{real}}(\cdot\mid s,a)}&\le \min\!\left\{1,\sqrt{\tfrac12\KL\!\left(\hat P_h^{t,\mathrm{real}}(\cdot\mid s,a),P_h^{\mathrm{real}}(\cdot\mid s,a)\right)}\right\}\\
&\le\rho_h^t(s,a).
\end{aligned}
\end{gathered}
\]
Hence, by the triangle inequality,
\begin{align*}
&\dtv{P_h^{\mathrm{real}}(\cdot \mid s,a)}{P_h^{\mathrm{sim}}(\cdot \mid s,a)} \\
\leq& \dtv{P_h^{\mathrm{real}}(\cdot \mid s,a)}{\hat P_h^{t,\mathrm{real}}(\cdot \mid s,a)} + \dtv{\hat P_h^{t,\mathrm{real}}(\cdot \mid s,a)}{P_h^{\mathrm{sim}}(\cdot \mid s,a)} \\
\leq& \rho_h^t(s,a) + \dtv{\hat P_h^{t,\mathrm{real}}(\cdot \mid s,a)}{P_h^{\mathrm{sim}}(\cdot \mid s,a)} \\
\leq& \frac{\sigma_s+\epsilon_s}{2}.
\end{align*}
However, since $(h,s,a) \in \mc B$, Assumption~\ref{ass:separation} implies $\dtv{P_h^{\mathrm{real}}(\cdot \mid s,a)}{P_h^{\mathrm{sim}}(\cdot \mid s,a)} \geq \sigma_s$, which is a contradiction. Therefore, $(h,s,a) \notin \mc G^t$. Since $\mc B \subseteq \hat{\mc B}^0$ by initialization, no element of $\mc B$ is ever removed from $\hat{\mc B}^t$, and thus $\mc B \subseteq \hat{\mc B}^t$ for all $t\in\mathbb{N}$.
\end{proof}
\begin{lemma}\label{lem:completeness}
On the event $\mc E$, for every $(h,s,a)\notin \mc B$ and every $t\in\mathbb N$, if $n_h^t(s,a)\ge n$, then $(h,s,a)\notin \hat{\mc B}^{t}$, where
\begin{align}
\!\!\!\!\!\!n := 1+\frac{16\left(\log(2|\mc S||\mc A|H/\delta)+|\mc S|\log(8e)\right)}{(\sigma_s-\epsilon_s)^2}+\frac{16|\mc S|}{(\sigma_s-\epsilon_s)^2}\log\!\left(\max\left\{\frac{16|\mc S|}{(\sigma_s-\epsilon_s)^2},e\right\}\right) .\label{eq_n_dagger}
\end{align}
\end{lemma}
\begin{proof}[Proof of Lemma~\ref{lem:completeness}]
Fix \((h,s,a)\notin\mc B\) and \(t\in\mathbb N\). Since
\((h,s,a)\notin\mc B\),
\[
\dtv{P_h^{\mathrm{real}}(\cdot\mid s,a)}
     {P_h^{\mathrm{sim}}(\cdot\mid s,a)}
\leq \epsilon_s.
\]
On the event \(\mc E\), Pinsker's inequality and the triangle inequality for total variation distance give
\begin{align}
&\dtv{\hat P_h^{t,\mathrm{real}}(\cdot\mid s,a)}{P_h^{\mathrm{sim}}(\cdot\mid s,a)}\\
\leq&\dtv{\hat P_h^{t,\mathrm{real}}(\cdot\mid s,a)}{P_h^{\mathrm{real}}(\cdot\mid s,a)}+\dtv{P_h^{\mathrm{real}}(\cdot\mid s,a)}{P_h^{\mathrm{sim}}(\cdot\mid s,a)} \\
\leq& \rho_h^t(s,a)+\epsilon_s.  
\end{align}
By Algorithm~\ref{alg:hybrid}, the triple belongs to \(\mc G^t\) whenever
\[
\dtv{\hat P_h^{t,\mathrm{real}}(\cdot\mid s,a)}
     {P_h^{\mathrm{sim}}(\cdot\mid s,a)}
+\rho_h^t(s,a)
\leq \frac{\sigma_s+\epsilon_s}{2}.
\]
Therefore, it suffices to ensure 
\[
\epsilon_s+2\rho_h^t(s,a) \leq \frac{\sigma_s+\epsilon_s}{2} \Longleftrightarrow \rho_h^t(s,a)\leq\frac{\sigma_s-\epsilon_s}{4}.
\]
Using the definiron of \(\rho_h^t(s,a)\), the above condition is equivalent to
\[
n_h^t(s,a)
\geq
\frac{8\beta(n_h^t(s,a),\delta)}
     {(\sigma_s-\epsilon_s)^2}.
\]
Applying Lemma~\ref{lem:log_threshold} with
\[
\kappa=\frac{(\sigma_s-\epsilon_s)^2}{8},
\qquad
b=|\mc S|,
\qquad
c=\log\left(\frac{2|\mc S||\mc A|H}{\delta}\right)
  +|\mc S|\log(8e),
\]
shows that \(n_h^t(s,a)\geq n\) implies
\[
n_h^t(s,a)
\geq
\frac{8\beta(n_h^t(s,a),\delta)}
     {(\sigma_s-\epsilon_s)^2}.
\]
Hence,
\[
\rho_h^t(s,a)\leq\frac{\sigma_s-\epsilon_s}{4},
\]
and the removal criterion in~\eqref{eq_deletetion_rule} is satisfied.

If \((h,s,a)\in\hat{\mc B}^{t-1}\), then
\((h,s,a)\in\mc G^t\), and therefore
\[
(h,s,a)\notin
\hat{\mc B}^t
=
\hat{\mc B}^{t-1}\setminus\mc G^t.
\]
If \((h,s,a)\notin\hat{\mc B}^{t-1}\), the monotonicity
\(\hat{\mc B}^t\subseteq\hat{\mc B}^{t-1}\) from (i) of
Lemma~\ref{lem:cover} gives the same conclusion.
\end{proof}

\begin{lemma}\label{lem:identification-budget}
On the event \(\mathcal E\cap\mathcal E^{\mathrm{cnt}}\), for every \((h,s,a)\notin\mc B\) and every \(T\in\mathbb N\),
\[
\sum_{t=1}^T\mathbf 1_{\{(h,s,a)\in\hat{\mc B}^t\}}d_h^{P^{\mathrm{real}},\pi^t}(s,a)\leq5n.
\]
\end{lemma}

\begin{proof}[Proof of Lemma~\ref{lem:identification-budget}]
Fix \((h,s,a)\notin\mc B\) and \(T\in\mathbb N\), and let
\(
\mathcal T:=\{t\in[T]:(h,s,a)\in\hat{\mc B}^t\}.
\)
If \(\mathcal T=\emptyset\), the claim holds trivially. Otherwise, Lemma~\ref{lem:cover} shows that the sets \(\hat{\mc B}^t\) are nonincreasing in \(t\). Thus, \(\mathcal T\) is a prefix of \([T]\). Let \(t^\star:=\max\mathcal T\). Since \((h,s,a)\in\hat{\mc B}^{t^\star}\) and \((h,s,a)\notin\mc B\), we have $n_h^{t^\star}(s,a)<n$ by Lemma~\ref{lem:completeness}. On the event \(\mathcal E^{\mathrm{cnt}}\), it follows that
\[
\bar n_h^{t^\star}(s,a)\leq2\left(n_h^{t^\star}(s,a)+\beta^{\cnt}(\delta)\right)<2\bigl(n+\beta^{\cnt}(\delta)\bigr).
\]
Finally, by the definition of the pseudo-count $n_{h}^{t}(s,a)$ and the bound \(d_h^{P^{\mathrm{real}},\pi^{t^\star}}(s,a)\leq1\),
\[
\sum_{t\in\mathcal T}d_h^{P^{\mathrm{real}},\pi^t}(s,a)\leq\bar n_h^{t^\star}(s,a)+d_h^{P^{\mathrm{real}},\pi^{t^\star}}(s,a)<2\bigl(n+\beta^{\cnt}(\delta)\bigr)+1\le 5n,
\]
where the last inequality uses \(\beta^{\cnt}(\delta)\leq n\) and \(n\geq1\). This proves the result.
\end{proof}

\subsection{Proof of Lemma \ref{lemma:upper_bound_on_estimation_error}}
\label{app:upper_bound_on_estimation_error}
\begin{proof}[Proof of Lemma~\ref{lemma:upper_bound_on_estimation_error}]
We condition on the event \(\mc E\), which holds with high probability by Lemma~\ref{lem:proba_master_event}. For every \((h,s,a)\in\hat{\mc B}^t\), Pinsker's inequality on the event \(\mc E\) gives
\[
\dtv{\hat P_h^{t}(\cdot\mid s,a)}{P_h^{\mathrm{real}}(\cdot\mid s,a)}=\dtv{\hat P_h^{t,\mathrm{real}}(\cdot\mid s,a)}{P_h^{\mathrm{real}}(\cdot\mid s,a)}\leq\rho_h^t(s,a).
\]
Moreover, Lemma~\ref{lem:cover} ensures that \(\mc B\subseteq\hat{\mc B}^t\). Together with Assumption \ref{ass:separation}, we have\(\dtv{\hat P_h^{t}(\cdot\mid s,a)}{P_h^{\mathrm{real}}(\cdot\mid s,a)}\leq\epsilon_s\) outside \(\hat{\mc B}^t\). Therefore, for every \((h,s,a)\),
\[
\dtv{\hat P_h^t(\cdot\mid s,a)}{P_h^{\mathrm{real}}(\cdot\mid s,a)}
\leq\rho_h^t(s,a)\mathbf1_{\{(h,s,a)\in\hat{\mathcal B}^t\}}
+\epsilon_s\mathbf1_{\{(h,s,a)\notin\hat{\mathcal B}^t\}}.
\]
Applying the simulation lemma~\citep[Lemma~4]{kalagarla2021sample} and
\(0\leq V_{f,h+1}^{P^{\mathrm{real}},\pi}\leq H-h\), we obtain
\begin{align}
e_{f,1}^{t,\pi}(s_1)&=\left|\sum_{h=1}^{H}\sum_{s,a}d_h^{\hat P^t,\pi}(s,a)\sum_{s'\in\mc S}\left(\hat P_h^t(s'\mid s,a)-P_h^{\mathrm{real}}(s'\mid s,a)\right)V_{f,h+1}^{P^{\mathrm{real}},\pi}(s')\right|\\
&\leq \sum_{h=1}^{H}\sum_{s,a}d_h^{\hat P^t,\pi}(s,a)b_h^t(s,a)=V_{b,1}^{\hat P^t,\pi}(s_1).
\end{align}
Furthermore, both value functions lie in \([0,H]\), so \(e_{f,1}^{t,\pi}(s_1)\leq H\). Combining the two bounds yields
\[
e_{f,1}^{t,\pi}(s_1)\leq\min\left\{V_{b,1}^{\hat P^t,\pi}(s_1),H\right\}.
\]
\end{proof}

\subsection{Proof of Theorem \ref{thm:main}}
\label{app:proof_main}
Before proceeding with the proof of Theorem~\ref{thm:main}, we state the following lemma, which shows that once the estimation error is sufficiently small, solving the empirical CMDP under a slightly tightened safety constraint yields a policy that is feasible and near-optimal for the true CMDP. We then use this lemma to prove Theorem~\ref{thm:main}.

\begin{lemma}[Connection between estimation error and CMDPs]
\label{lemma:Connection between estimation error and CMDPs}

Suppose Assumption~\ref{ass:slater} holds. At iteration $t$, assume that there exists a constant $\tau\in(0,\xi/4]$ such that
\begin{align}
V_{c,1}^{\hat P^t,\pi^0}(s_1)\ge \ell+ \frac{\xi}{2}\;\text{  and  }\max_{\pi\in\Pi_{\mathrm{feas}}^{\hat P^t}}V_{b,1}^{\hat P^t,\pi}(s_1)\le \tau.
\label{eq:max_error_condition}
\end{align}
Then, on the high-probability event of Lemma~\ref{lemma:upper_bound_on_estimation_error}, for any reward function $r\in\mathcal F$, every optimal solution
\[
\pi^{\mathrm{out}}
\in
\arg\max_{\pi\in\Pi}
\Bigl\{
V_{r,1}^{\hat P^t,\pi}(s_1)
:
V_{c,1}^{\hat P^t,\pi}(s_1)\ge \ell+\tau
\Bigr\}
\]
is feasible for the true CMDP and satisfies
\begin{align*}
V_{r,1}^{P^{\mathrm{real}},\pi^\star}(s_1)
-
V_{r,1}^{P^{\mathrm{real}},\pi^{\mathrm{out}}}(s_1)
\le
\frac{2H\tau}{\xi}
+
2\tau,
\end{align*}
where \(\pi^\star\in\arg\max_{\pi\in\Pi_{\mathrm{feas}}^{P^{\mathrm{real}}}}V_{r,1}^{P^{\mathrm{real}},\pi}(s_1)\).
\end{lemma}
\begin{proof}[Proof of Lemma~\ref{lemma:Connection between estimation error and CMDPs}]
We work on the high-probability event of Lemma~\ref{lemma:upper_bound_on_estimation_error}. On this event, for every \(\pi\in\Pi\) and \(f\in\mc F\),
\[
\left|V_{f,1}^{\hat P^t,\pi}(s_1)-V_{f,1}^{P^{\mathrm{real}},\pi}(s_1)\right|\leq V_{b,1}^{\hat P^t,\pi}(s_1).
\]
Since \(\max_{\pi\in\Pi_{\mathrm{feas}}^{\hat P^t}}V_{b,1}^{\hat P^t,\pi}(s_1)\leq\tau\), it follows that
\begin{equation}
\max_{f\in\mc F,\pi\in\Pi_{\mathrm{feas}}^{\hat P^t}}\left|V_{f,1}^{\hat P^t,\pi}(s_1)-V_{f,1}^{P^{\mathrm{real}},\pi}(s_1)\right|\leq\tau.
\label{eq:emp-feasible-error-bound}
\end{equation}

By the assumptions of the lemma, the baseline policy satisfies
\begin{equation}
V_{c,1}^{\hat P^t,\pi^0}(s_1)\geq\ell+\frac{\xi}{2},\qquad V_{c,1}^{P^{\mathrm{real}},\pi^0}(s_1)\geq\ell+\xi.
\label{eq:true-slater-pi0}
\end{equation}
Let
\(
\pi^\star\in\arg\max_{\pi\in\Pi_{\mathrm{feas}}^{P^{\mathrm{real}}}}V_{r,1}^{P^{\mathrm{real}},\pi}(s_1).
\)
For \(\lambda\in[0,1]\), we define \(\pi^\lambda = \lambda \pi^0 + (1-\lambda) \pi^\star\). For every \(f\in\mc F\) and \(P\in\{\hat P^t,P^{\mathrm{real}}\}\), we have
\begin{equation}
V_{f,1}^{P,\pi^\lambda}(s_1)=(1-\lambda)V_{f,1}^{P,\pi^\star}(s_1)+\lambda V_{f,1}^{P,\pi^0}(s_1).
\label{eq:mixture-linearity}
\end{equation}
Consequently,
\[
\max_{f\in\mc F}\left|V_{f,1}^{\hat P^t,\pi^\lambda}(s_1)-V_{f,1}^{P^{\mathrm{real}},\pi^\lambda}(s_1)\right|\leq(1-\lambda)V_{b,1}^{\hat P^t,\pi^\star}(s_1)+\lambda V_{b,1}^{\hat P^t,\pi^0}(s_1)=V_{b,1}^{\hat P^t,\pi^\lambda}(s_1).
\]

To relate this mixture to a Markov policy, we define
\[
d_h^\lambda(s,a):=(1-\lambda)d_h^{\hat P^t,\pi^\star}(s,a)+\lambda d_h^{\hat P^t,\pi^0}(s,a)
\]
and construct \(\tilde\pi\in\Pi\) by
\[
\tilde\pi_h(a\mid s):=\frac{d_h^\lambda(s,a)}{\sum_{a'\in\mc A}d_h^\lambda(s,a')}
\]
whenever the denominator is positive, choosing \(\tilde\pi_h(\cdot\mid s)\) arbitrarily otherwise. Since the occupancy-measure constraints are linear, \(d^\lambda\) is valid under \(\hat P^t\) and is induced by \(\tilde\pi\). Thus, for every \(u\in\mc F\cup\{b\}\),
\[
V_{u,1}^{\hat P^t,\tilde\pi}(s_1)=V_{u,1}^{\hat P^t,\pi^\lambda}(s_1).
\]
Therefore, \(\pi^\lambda \in \Pi_{\mathrm{feas}}^{\hat P^t}\) impies
\[
V_{b,1}^{\hat P^t,\pi^\lambda}(s_1)=V_{b,1}^{\hat P^t,\tilde\pi}(s_1)\leq\tau.
\]

Since \(\tau\leq\xi/2\), Equation~\eqref{eq:true-slater-pi0} implies \(V_{c,1}^{\hat P^t,\pi^0}(s_1)\geq\ell+\tau\). Hence,
\[
\Lambda:=\left\{\lambda\in[0,1]:V_{c,1}^{\hat P^t,\pi^\lambda}(s_1)\geq\ell+\tau\right\}
\]
is nonempty. Define \(\lambda_\star:=\inf\Lambda\). By continuity of \(\lambda\mapsto V_{c,1}^{\hat P^t,\pi^\lambda}(s_1)\), the set \(\Lambda\) is closed, so \(\lambda_\star\in\Lambda\). Therefore,
\begin{equation}
V_{c,1}^{\hat P^t,\pi^{\lambda_\star}}(s_1)\geq\ell+\tau.
\label{eq:lambda-star-emp-feasible}
\end{equation}

If \(\lambda_\star=0\), then \(\lambda_\star\leq2\tau/\xi\) holds trivially. Otherwise, minimality and continuity imply
\[
V_{c,1}^{\hat P^t,\pi^{\lambda_\star}}(s_1)=\ell+\tau.
\]
Since \(\pi^{\lambda_\star}\in\Pi_{\mathrm{feas}}^{\hat P^t}\), we have
\[
V_{c,1}^{P^{\mathrm{real}},\pi^{\lambda_\star}}(s_1)\leq V_{c,1}^{\hat P^t,\pi^{\lambda_\star}}(s_1)+\tau=\ell+2\tau.
\]
On the other hand, \(\pi^\star\in\Pi_{\mathrm{feas}}^{P^{\mathrm{real}}}\) and Equation~\eqref{eq:true-slater-pi0} give
\[
V_{c,1}^{P^{\mathrm{real}},\pi^{\lambda_\star}}(s_1)=(1-\lambda_\star)V_{c,1}^{P^{\mathrm{real}},\pi^\star}(s_1)+\lambda_\star V_{c,1}^{P^{\mathrm{real}},\pi^0}(s_1)\ge (1-\lambda_\star)\ell+\lambda_\star(\ell+\xi)=\ell+\lambda_\star\xi.
\]
Combining these inequalities yields
\begin{equation}
\lambda_\star\leq\frac{2\tau}{\xi}.
\label{eq:lambda-star-small}
\end{equation}

We next prove feasibility of \(\pi^{\mathrm{out}}\). Since it is feasible for the tightened empirical CMDP, we have 
\(
V_{c,1}^{\hat P^t,\pi^{\mathrm{out}}}(s_1)\geq\ell+\tau.
\)
Aplying Equation~\eqref{eq:emp-feasible-error-bound} gives
\[
V_{c,1}^{P^{\mathrm{real}},\pi^{\mathrm{out}}}(s_1)\geq V_{c,1}^{\hat P^t,\pi^{\mathrm{out}}}(s_1)-\tau\geq\ell.
\]
Thus, \(\pi^{\mathrm{out}}\) is feasible for the true CMDP.

Finally, Equation~\eqref{eq:lambda-star-emp-feasible} and the occupancy-measure construction imply that \(\tilde\pi\) is feasible for the tightened empirical CMDP. By optimality of \(\pi^{\mathrm{out}}\),
\[
V_{r,1}^{\hat P^t,\pi^{\mathrm{out}}}(s_1)\geq V_{r,1}^{\hat P^t,\tilde\pi}(s_1)=V_{r,1}^{\hat P^t,\pi^{\lambda_\star}}(s_1).
\]
Applying Equation~\eqref{eq:emp-feasible-error-bound} to \(\pi^{\mathrm{out}}\), we obtain
\begin{equation}
V_{r,1}^{P^{\mathrm{real}},\pi^{\mathrm{out}}}(s_1)\geq V_{r,1}^{P^{\mathrm{real}},\pi^{\lambda_\star}}(s_1)-2\tau.
\label{eq:reward-transfer-bound}
\end{equation}
Since rewards are nonnegative and bounded by \(1\),
\[
V_{r,1}^{P^{\mathrm{real}},\pi^{\lambda_\star}}(s_1)\geq(1-\lambda_\star)V_{r,1}^{P^{\mathrm{real}},\pi^\star}(s_1)
\]
and \(V_{r,1}^{P^{\mathrm{real}},\pi^\star}(s_1)\leq H\). Therefore, Equation~\eqref{eq:lambda-star-small} gives
\[
V_{r,1}^{P^{\mathrm{real}},\pi^\star}(s_1)-V_{r,1}^{P^{\mathrm{real}},\pi^{\lambda_\star}}(s_1)\leq\lambda_\star H\leq\frac{2H\tau}{\xi}.
\]
Combining this inequality with Equation~\eqref{eq:reward-transfer-bound} yields
\[
V_{r,1}^{P^{\mathrm{real}},\pi^\star}(s_1)-V_{r,1}^{P^{\mathrm{real}},\pi^{\mathrm{out}}}(s_1)\leq\frac{2H\tau}{\xi}+2\tau.
\]
This completes the proof.
\end{proof}

Now, we are ready to prove Theorem~\ref{thm:main}.

\begin{proof}[Proof of Theorem~\ref{thm:main}]
Below, we condition on \(\mathcal E\cap\mathcal E^{\mathrm{cnt}}\), which holds with probability at least \(1-\delta\). We first prove the sample-complexity bound and then establish safe exploration and planning accuracy.

For this proof only, define the statistical bonus, which is supported on \(\hat{\mc B}^t\), by
\[
b_h^{t,\mathrm{stat}}(s,a)\coloneqq H\rho_h^t(s,a)\mathbf 1_{\{(h,s,a)\in\hat{\mc B}^t\}},
\]
and define the transition kernel
\[
P_h^{\dagger,t}(\cdot\mid s,a)\coloneqq\hat P_h^{t,\mathrm{real}}(\cdot\mid s,a)\mathbf 1_{\{(h,s,a)\in\hat{\mc B}^t\}}+P_h^{\mathrm{real}}(\cdot\mid s,a)\mathbf 1_{\{(h,s,a)\notin\hat{\mc B}^t\}}.
\]
By the triangle inequality and the simulation lemma~\citep[Lemma~4]{kalagarla2021sample}, for every \(\pi\in\Pi\),
\begin{align}
&\left|V_{b,1}^{\hat P^t,\pi}(s_1)-V_{b^{t,\mathrm{stat}},1}^{P^{\dagger,t},\pi}(s_1)\right|
\leq\left|V_{b,1}^{\hat P^t,\pi}(s_1)-V_{b^{t,\mathrm{stat}},1}^{\hat P^t,\pi}(s_1)\right|+\left|V_{b^{t,\mathrm{stat}},1}^{\hat P^t,\pi}(s_1)-V_{b^{t,\mathrm{stat}},1}^{P^{\dagger,t},\pi}(s_1)\right|\\
&\leq H^2\epsilon_s+\left|\sum_{h=1}^H\sum_{s,a}d_h^{\hat P^t,\pi}(s,a)\sum_{s'\in\mc S}\left(\hat P_h^t(s'\mid s,a)-P_h^{\dagger,t}(s'\mid s,a)\right)V_{b^{t,\mathrm{stat}},h+1}^{P^{\dagger,t},\pi}(s')\right|\\
&\leq H^2\epsilon_s+H\epsilon_s\sum_{h=1}^H(H-h)\leq H^3\epsilon_s.
\label{eq_upper_on_b}
\end{align}

We next lower-bound \(V_{b^{t,\mathrm{stat}},1}^{P^{\dagger,t},\pi^t}(s_1)\) for every executed iteration. We consider two cases. First, suppose the condition in Line~5 fails, so the algorithm executes the mixture policy \(\pi^t\). Since the algorithm has not terminated at iteration \(t<T\), the stopping condition gives
\[
V_{b,1}^{\hat P^t,\bar\pi^t}(s_1)>\tau.
\]
Inequality~\eqref{eq_upper_on_b} and the definition of \(\alpha^t\) yield
\begin{align}
V_{b^{t,\mathrm{stat}},1}^{P^{\dagger,t},\pi^t}(s_1)
&=\alpha^tV_{b^{t,\mathrm{stat}},1}^{P^{\dagger,t},\bar\pi^t}(s_1)+(1-\alpha^t)V_{b^{t,\mathrm{stat}},1}^{P^{\dagger,t},\pi^0}(s_1)
\geq\alpha^t\left(V_{b,1}^{\hat P^t,\bar\pi^t}(s_1)-H^3\epsilon_s\right)\\
&\geq\frac{\xi\left(V_{b,1}^{\hat P^t,\bar\pi^t}(s_1)-H^3\epsilon_s\right)}{\xi+V_{b,1}^{\hat P^t,\bar\pi^t}(s_1)}
\geq\frac{\xi(\tau-H^3\epsilon_s)}{\xi+\tau}.
\end{align}
Using \(\tau-H^3\epsilon_s=\xi\epsilon/(4H)\) and \(\tau\leq\xi/4\), we obtain
\begin{equation}
V_{b^{t,\mathrm{stat}},1}^{P^{\dagger,t},\pi^t}(s_1)\geq\frac{\xi\epsilon}{5H}.
\label{eq_lowerbound}
\end{equation}

Second, suppose the condition in Line~5 holds. Then
\[
V_{c,1}^{\hat P^t,\pi^0}(s_1)<\ell+\frac{\xi}{2},
\]
and the algorithm sets \(\pi^t=\pi^0\). Assumption~\ref{ass:slater} and Lemma~\ref{lemma:upper_bound_on_estimation_error} imply
\begin{align}
V_{b,1}^{\hat P^t,\pi^t}(s_1)
=V_{b,1}^{\hat P^t,\pi^0}(s_1)
\geq\left|V_{c,1}^{P^{\mathrm{real}},\pi^0}(s_1)-V_{c,1}^{\hat P^t,\pi^0}(s_1)\right|
>\frac{\xi}{2}.
\end{align}
Applying Inequality~\eqref{eq_upper_on_b} gives
\[
V_{b^{t,\mathrm{stat}},1}^{P^{\dagger,t},\pi^t}(s_1)\geq V_{b,1}^{\hat P^t,\pi^t}(s_1)-H^3\epsilon_s\geq\frac{\xi}{2}-H^3\epsilon_s\geq\frac{\xi\epsilon}{5H},
\]
where the final inequality uses \(\epsilon_s\leq\xi/(4H^3)\) and \(\epsilon\leq1\).

Thus, Inequality~\eqref{eq_lowerbound} holds for every executed iteration. Summing over these iterations yields
\begin{align}
    \sum_{t=1}^T V_{b^{t,\mathrm{stat}},1}^{P^{\dagger,t},\pi^t}(s_1)\geq\frac{\xi\epsilon T}{5H}.\label{eq_lower_sum}
\end{align}
Next, we derive an upper bound on \(V_{b^{t,\mathrm{stat}},1}^{P^{\dagger,t},\pi}(s_1)\). We first show that, for every \((h,s,a)\in[H]\times\mc S\times\mc A\) and every \(w:\mc S\to[0,H]\),
\begin{equation}\label{eq:one-step-transport}
\begin{gathered}
{\begin{aligned}
\sum_{s'\in\mc S} P_h^{\dagger,t}(s'\mid s,a)w(s')
\le{}&{2H^2}(\rho_h^t(s,a))^2\mathbf1_{\{(h,s,a)\in\hat{\mc B}^t\}}+\left(1+\frac1H\right)\sum_{s'\in\mc S}P_h^{\mathrm{real}}(s'\mid s,a)w(s').
\end{aligned}}
\end{gathered}
\end{equation}
This inequality follows by considering three cases. First, if \((h,s,a)\notin\hat{\mc B}^t\), then {\(P_h^{\dagger,t}(\cdot\mid s,a)=P_h^{\mathrm{real}}(\cdot\mid s,a)\)}, so the inequality holds immediately. Second, if \((h,s,a)\in\hat{\mc B}^t\) and \(\rho_h^t(s,a)=1\), then the left-hand side is at most {\(H\leq2H^2\)}. Third, suppose that \((h,s,a)\in\hat{\mc B}^t\) and \(\rho_h^t(s,a)<1\). Applying Lemma~\ref{lem:donsker-varadhan} with {\(q=P_h^{\dagger,t}(\cdot\mid s,a)=\hat P_h^{t,\mathrm{real}}(\cdot\mid s,a)\)}, \(p=P_h^{\mathrm{real}}(\cdot\mid s,a)\), and \(\lambda=1/H^2\) gives\looseness-1
\begin{align}
   &\sum_{s'\in\mc S}\hat P_h^{t,\mathrm{real}}(s'\mid s,a)w(s')\\
   \leq& H^2\KL\left(\hat P_h^{t,\mathrm{real}}(\cdot\mid s,a),P_h^{\mathrm{real}}(\cdot\mid s,a)\right)+H^2\log\left(\sum_{s'\in\mc S}P_h^{\mathrm{real}}(s'\mid s,a)e^{w(s')/H^2}\right)\\
   \le& 2H^2\bigl(\rho_h^t(s,a)\bigr)^2 +H^2\log\left(\sum_{s'\in\mc S}P_h^{\mathrm{real}}(s'\mid s,a)e^{w(s')/H^2}\right),\label{eq:dv-applied}
\end{align}
where the last inequality follows from event \(\mc E\). Meanwhile, the convexity of \(z\mapsto e^{\lambda z}\) on \([0,H]\) gives
\[
e^{\lambda z}\leq1+\frac{z}{H}\left(e^{\lambda H}-1\right),\,\forall z\in[0,H].
\]
Setting \(\lambda=1/H^2\) and \(z=w(s')\), we obtain
\[
\sum_{s'\in\mc S}P_h^{\mathrm{real}}(s'\mid s,a)e^{w(s')/H^2}\leq1+\frac{e^{1/H}-1}{H}\sum_{s'\in\mc S}P_h^{\mathrm{real}}(s'\mid s,a)w(s').
\]
Using \(\log(1+y)\leq y\) for \(y\geq0\), together with \(e^y\leq1+y+y^2\) for \(y\in[0,1]\) and \(y=1/H\), yields
\begin{align}
H^2\log\left(\sum_{s'\in\mc S}P_h^{\mathrm{real}}(s'\mid s,a)e^{w(s')/H^2}\right)&\leq H\left(e^{1/H}-1\right)\sum_{s'\in\mc S}P_h^{\mathrm{real}}(s'\mid s,a)w(s')\\
&\leq\left(1+\frac{1}{H}\right)\sum_{s'\in\mc S}P_h^{\mathrm{real}}(s'\mid s,a)w(s').
\end{align}
Substituting this bound into Inequality~\eqref{eq:dv-applied} proves Inequality~\eqref{eq:one-step-transport}.

Expanding \(V_{b^{t,\mathrm{stat}},1}^{P^{\dagger,t},\pi^t}(s_1)\) into its bonus terms, recursively applying Inequality~\eqref{eq:one-step-transport} to each term with $w(s) = V_{b^{t,\mathrm{stat}},h+1}^{P^{\dagger,t},\pi^t}(s)/H$, using \((1+1/H)^H\leq e\), and multiplying the resulting inequality through by $H$, we obtain

\begin{equation}\label{eq:linear-transport}
\begin{split}
V_{b^{t,\mathrm{stat}},1}^{P^{\dagger,t},\pi^t}(s_1)\le
\underbrace{e\sum_{h,s,a}d_h^{P^{\mathrm{real}},\pi^t}(s,a)\,b_h^{t,\mathrm{stat}}(s,a)}_{(i)^t}
+\underbrace{2eH\sum_{h,s,a}d_h^{P^{\mathrm{real}},\pi^t}(s,a)
\bigl(b_h^{t,\mathrm{stat}}(s,a)\bigr)^2}_{(ii)^t}.
\end{split}
\end{equation}
We first bound Term~$(i)^t$. For every \((h,s,a)\in\mc B\), Lemma~\ref{lem:cover} ensures that \((h,s,a)\in\hat{\mc B}^t\), and hence \(b_h^{t,\mathrm{stat}}(s,a)=H\rho_h^t(s,a)\). Therefore,
\begin{align}
\sum_{t=1}^T d_h^{P^{\mathrm{real}},\pi^t}(s,a)b_h^{t,\mathrm{stat}}(s,a)&=H\sum_{t=1}^T d_h^{P^{\mathrm{real}},\pi^t}(s,a)\rho_h^t(s,a)\\
&\overset{(1)}{\leq}H\sum_{t=1}^T d_h^{P^{\mathrm{real}},\pi^t}(s,a)\min\left\{1,\sqrt{\frac{4\beta(\bar n_h^t(s,a),\delta)}{\bar n_h^t(s,a)\vee1}}\right\}\\
&\overset{(2)}{\leq}4\sqrt{2}H\sqrt{\beta(\bar n_h^{T+1}(s,a),\delta)\bar n_h^{T+1}(s,a)}\\
&\overset{(3)}{\leq}4\sqrt{2}H\sqrt{\beta(T,\delta)\bar n_h^{T+1}(s,a)},
\end{align}
where Step~(1) follows from Lemma~\ref{lem:radius-pseudocount}, Step~(2) follows from Lemma~\ref{lem:radius-counting}, and Step~(3) uses the monotonicity of \(\beta(\cdot,\delta)\) and \(\bar n_h^{T+1}(s,a)\leq T\).

Summing over \((h,s,a)\in\mc B\) and applying the Cauchy--Schwarz inequality gives
\begin{align}
\sum_{(h,s,a)\in\mc B}\sum_{t=1}^T d_h^{P^{\mathrm{real}},\pi^t}(s,a)b_h^{t,\mathrm{stat}}(s,a)&\leq4\sqrt{2}H\sqrt{\beta(T,\delta)|\mc B|}\sqrt{\sum_{(h,s,a)\in\mc B}\bar n_h^{T+1}(s,a)}\\
&\leq4\sqrt{2}H\sqrt{\beta(T,\delta)|\mc B|HT},
\end{align}
where the last inequality uses
\[
\sum_{(h,s,a)\in\mc B}\bar n_h^{T+1}(s,a)\leq\sum_{t=1}^T\sum_{h,s,a}d_h^{P^{\mathrm{real}},\pi^t}(s,a)=HT.
\]

Next, fix \((h,s,a)\notin\mc B\) and define
\[
\mathcal T_{h,s,a}:=\{t\in[T]:(h,s,a)\in\hat{\mc B}^t\}.
\]
If \(\mathcal T_{h,s,a}=\emptyset\), it contributes 0 to $(i)^t$. Otherwise, let \(t^\star:=\max\mathcal T_{h,s,a}\). Since \(\mathcal T_{h,s,a}\) is a prefix of \([T]\), applying Lemmas~\ref{lem:radius-pseudocount} and~\ref{lem:radius-counting} up to \(t^\star\) yields
\begin{align}
\sum_{t=1}^T d_h^{P^{\mathrm{real}},\pi^t}(s,a)b_h^{t,\mathrm{stat}}(s,a)&=H\sum_{t\in\mathcal T_{h,s,a}}d_h^{P^{\mathrm{real}},\pi^t}(s,a)\rho_h^t(s,a)\\
&\leq4\sqrt{2}H\sqrt{\beta(\bar n_h^{t^\star+1}(s,a),\delta)\bar n_h^{t^\star+1}(s,a)}.
\end{align}
By Lemma~\ref{lem:identification-budget},
\[
\bar n_h^{t^\star+1}(s,a)=\sum_{t\in\mathcal T_{h,s,a}}d_h^{P^{\mathrm{real}},\pi^t}(s,a)\leq5n,
\]
where the last inequality uses \(\beta^{\cnt}(\delta)\leq n\) and \(n\geq1\). Consequently,
\[
\sum_{t=1}^T d_h^{P^{\mathrm{real}},\pi^t}(s,a)b_h^{t,\mathrm{stat}}(s,a)\leq4\sqrt{2}H\sqrt{5\beta(5n,\delta)n}.
\]
Summing this bound over all non-mismatch triples gives
\[
\sum_{(h,s,a)\notin\mc B}\sum_{t=1}^T d_h^{P^{\mathrm{real}},\pi^t}(s,a)b_h^{t,\mathrm{stat}}(s,a)\leq4\sqrt{2}H\bigl(H|\mc S||\mc A|-|\mc B|\bigr)\sqrt{5\beta(5n,\delta)n}.
\]

Combining the bounds for mismatch and non-mismatch triples, we obtain
\begin{align}
\sum_{t=1}^T(i)^t\leq4\sqrt{2}eH\left(\sqrt{\beta(T,\delta)|\mc B|HT}+\bigl(H|\mc S||\mc A|-|\mc B|\bigr)\sqrt{5\beta(5n,\delta)n}\right).\label{eq:radius-sum}
\end{align}

Bounding the second term \(\sum_{t=1}^T(ii)^t\) is similar as bounding the first term \(\sum_{t=1}^T(i)^t\), using the second bound of Lemma~\ref{lem:radius-counting} instead of the first, we have
\begin{align}
\sum_{t=1}^T(ii)^t&\leq16eH^3\bigl[|\mc B|\Phi(T)+(H|\mc S||\mc A|-|\mc B|)\Phi(5n)\bigr],\label{eq:squared-radius-sum}
\end{align}
where $\Phi(z)=\beta(z,\delta)\bigl[1+\log_+(z/4\beta(z,\delta))\bigr]$.

Combining Inequalities \eqref{eq_lower_sum}, \eqref{eq:radius-sum} and~\eqref{eq:squared-radius-sum} gives
\begin{equation}
    \frac{\xi\epsilon T}{5H}\le4\sqrt{2}eH\bigl(\sqrt{\beta(T,\delta)|\mc B|HT}+N_0\sqrt{5\beta(5n,\delta)n}\bigr)+16eH^3\bigl[|\mc B|\Phi(T)+N_0\Phi(5n)\bigr].
\end{equation}
where $M\coloneqq H|\mc S||\mc A|$ and $N_0\coloneqq M-|\mc B|$. For convenience, define
\begin{align}
\eta&\coloneqq\frac{\xi\epsilon}{5H},\qquad A_T\coloneqq4\sqrt{2}eH\sqrt{\beta(T,\delta)|\mc B|H},\\
R_T&\coloneqq4\sqrt{2}eHN_0\sqrt{5\beta(5n,\delta)n}+16eH^3\left[|\mc B|\Phi(T)+N_0\Phi(5n)\right].
\end{align}
The preceding inequality can then be written as
\begin{equation}\label{eq:implicit-T}
\eta T\leq A_T\sqrt{T}+R_T.
\end{equation}
By Young's inequality,
\[
A_T\sqrt{T}\leq\frac{\eta}{2}T+\frac{A_T^2}{2\eta}.
\]
Substituting this bound into Inequality~\eqref{eq:implicit-T} and rearranging gives
\[
T\leq\frac{A_T^2}{\eta^2}+\frac{2R_T}{\eta}.
\]
Since \(A_T^2=32e^2H^3|\mc B|\beta(T,\delta)\), we obtain
\begin{equation}\label{eq:finite-stopping-bound}
T\leq\frac{32e^2H^3|\mc B|\beta(T,\delta)}{\eta^2}+\frac{8\sqrt{2}eHN_0\sqrt{5\beta(5n,\delta)n}}{\eta}+\frac{32eH^3\left[|\mc B|\Phi(T)+N_0\Phi(5n)\right]}{\eta}.
\end{equation}

We also derive an alternative bound that does not rely on the removal of non-mismatch triples. Specifically, applying the same argument to Terms~\((i)^t\) and~\((ii)^t\) while treating all \(M=H|\mc S||\mc A|\) triples as mismatch triples yields
\begin{equation}\label{eq:finite-online-bound}
T\leq\frac{32e^2H^3M\beta(T,\delta)}{\eta^2}+\frac{32eH^3M\Phi(T)}{\eta}.
\end{equation}

It remains to resolve the dependence on \(T\) on the right-hand sides of these bounds. Since
\[
\beta(T,\delta)=\log\left(\frac{2|\mc S||\mc A|H}{\delta}\right)+|\mc S|\log(8e(T+1)),
\]
grows linearly in \(\log(e(T+1))\), whereas
\[
\Phi(T)=\beta(T,\delta)\left[1+\log_+\left(\frac{T}{4\beta(T,\delta)}\right)\right]
\]
grows at most quadratically in \(\log(e(T+1))\). Hence, the right-hand side of each of the bounds in~\eqref{eq:finite-online-bound} and~\eqref{eq:finite-stopping-bound} is at most \(C\log^2(e(T+1))\) for some \(C\geq1\) that depends on the problem parameters but not on \(T\). Lemma~\ref{lem:self-bounding} therefore implies
\[
T\leq64C\log^2(64eC).
\]
Thus, the algorithm terminates after finitely many iterations, and the remaining dependence on \(T\) contributes only logarithmic factors.In particular,
\[
\beta(T,\delta)=\widetilde{\mc O}(|\mc S|),\qquad \Phi(T)=\widetilde{\mc O}(|\mc S|).
\]

By Equation~\eqref{eq_n_dagger}, we obtain
\[
n=\widetilde{\mc O}\left(\frac{|\mc S|}{(\sigma_s-\epsilon_s)^2}\right),\qquad \sqrt{\beta(5n,\delta)n}=\widetilde{\mc O}\left(\frac{|\mc S|}{\sigma_s-\epsilon_s}\right),\qquad \Phi(5n)=\widetilde{\mc O}(|\mc S|).
\]
Moreover,
\(
\frac{1}{\eta}=\mc O\left(\frac{H}{\xi\epsilon}\right).
\)
Substituting these estimates into~\eqref{eq:finite-stopping-bound} gives
\begin{align}
T&\leq\widetilde{\mc O}\left(\frac{H^5|\mc S||\mc B|}{\xi^2\epsilon^2}+\frac{H^2|\mc S|N_0}{\xi\epsilon(\sigma_s-\epsilon_s)}+\frac{H^4|\mc S|M}{\xi\epsilon}\right) \\
&=\widetilde{\mc O}\left(\frac{H^4|\mc S||\mc B|}{\xi\epsilon}\left(\frac{H}{\xi\epsilon}+1\right)+\frac{H^2|\mc S|(H|\mc S||\mc A|-|\mc B|)}{\xi\epsilon}\left(\frac{1}{(\sigma_s-\epsilon_s)}+{H^2}\right)\right) \\
&= \widetilde{\mc O}\left(\frac{H^5|\mc S||\mc B|}{\xi^2\epsilon^2}+\frac{H^4|\mc S|(H|\mc S||\mc A|-|\mc B|)}{\xi\epsilon}\left(\frac{1}{(\sigma_s-\epsilon_s)H^2}+1\right)\right),
\end{align}
where first equality uses defintion of \(M\) and $N_0$ and second equality uses \(\epsilon\leq1\) and \(\xi\leq H\).

Since each iteration contains \(H\) real-world samples, the total sample complexity satisfies
\begin{equation}\label{eq:HT-bound}
HT\leq\widetilde{\mc O}\left(\frac{H^6|\mc S||\mc B|}{\xi^2\epsilon^2}+\frac{H^5|\mc S|(H|\mc S||\mc A|-|\mc B|)}{\xi\epsilon}\left(\frac{1}{(\sigma_s-\epsilon_s)H^2}+1\right)\right).
\end{equation}
Similarly, substituting these estimates into~\eqref{eq:finite-online-bound} gives
\begin{equation}\label{eq:HT-online}
HT\leq\widetilde{\mc O}\left(\frac{H^7|\mc S|^2|\mc A|}{\xi^2\epsilon^2}\right).
\end{equation}

Both bounds hold for the same algorithm, so we may take their minimum and conclude that the total sample complexity satisfies
\begin{align}
    HT &\leq \widetilde{\mc O}\!\biggl(\min\biggl\{\frac{H^7|\mc S|^2|\mc A|}{\xi^2\epsilon^2},\underbrace{\frac{H^6|\mc S||\mc B|}{\xi^2\epsilon^2}}_{(i)}+\underbrace{\frac{H^5|\mc S|\bigl(H|\mc S||\mc A|-|\mc B|\bigr)}{\xi\epsilon}\left(1+\frac{1}{(\sigma_s-\epsilon_s)H^2}\right)}_{(ii)}\biggr\}\biggr).
\end{align}
\paragraph{Safe exploration.}
If \(V_{c,1}^{\hat P^t,\pi^0}(s_1)<\ell+\xi/2\), we execute \(\pi^0\), which is safe by Assumption~\ref{ass:slater}. Otherwise, we execute the mixture policy \(\pi^t\), which satisfies
\begin{align*}
V_{c,1}^{P^{\mathrm{real}},\pi^t}
&=\alpha^t V_{c,1}^{P^{\mathrm{real}},\bar\pi^t}+(1-\alpha^t)V_{c,1}^{P^{\mathrm{real}},\pi^0}\\
&\overset{(i)}{\geq}\alpha^t\left(V_{c,1}^{\hat P^t,\bar\pi^t}-\min\left\{H,V_{b,1}^{\hat P^t,\bar\pi^t}(s_1)\right\}\right)+(1-\alpha^t)(\xi+\ell)\\
&=\xi+\ell-\alpha^t\left(\min\left\{H,V_{b,1}^{\hat P^t,\bar\pi^t}(s_1)\right\}+\ell-V_{c,1}^{\hat P^t,\bar\pi^t}(s_1)+\xi\right)\\
&\overset{(ii)}{\geq}\ell.
\end{align*}
Here, step~(i) follows from Assumption~\ref{ass:slater} and Lemma~\ref{lemma:upper_bound_on_estimation_error}, while step~(ii) follows from the definition of \(\alpha^t\). Therefore, safe exploration is guaranteed.
\paragraph{Planning accuracy.}
Algorithm~\ref{alg:hybrid} terminates when \(V_{c,1}^{\hat P^t,\pi^0}(s_1)\geq\ell+\xi/2\), and \(\epsilon\) is chosen such that \(\tau\leq\xi/4\). Therefore, by Lemma~\ref{lemma:Connection between estimation error and CMDPs}, the output policy \(\pi^{\mathrm{out}}\in\Pi_{\mathrm{feas}}^{P^{\mathrm{real}}}\) and satisfies
\begin{align}\label{eq:planning-integrated-bound}
V_{r,1}^{P^{\mathrm{real}},\pi^\star}(s_1)-V_{r,1}^{P^{\mathrm{real}},\pi^{\mathrm{out}}}(s_1)
\leq\left(2+\frac{2H}{\xi}\right)\tau
=\frac{\xi+H}{2H}\epsilon
+\left(2+\frac{2H}{\xi}\right)H^3\epsilon_s
\leq\epsilon+\frac{4H^4}{\xi}\epsilon_s,
\end{align}
where the final inequality uses \(\xi\leq H\).
\end{proof}

\section{Supporting lemmas}
We first state a deviation inequality for empirical distributions of i.i.d.\ categorical samples, which will be used to control the estimation error of empirical transition models. Let $(X_t)_{t\in\mathbb{N}}$ be i.i.d.\ samples from a distribution supported on $[m]$ with probability vector $p\in\Delta([m])$. We denote by $\hat p_n$ the empirical distribution, i.e.,
\[
 \hat p_{n,k} = \frac{1}{n} \sum_{\ell=1}^n \mathbf{1}_{\{X_\ell = k\}},\, k\in[m].
\]

\begin{lemma}[{\cite[Proposition~1]{jonsson2020planning}}]
\label{lemma:max_ineq_categorical}
For all $p\in\Delta([m])$ and all $\delta\in(0,1)$,
\begin{align*}
     \Pr\left(\exists n\in \mathbb{N},\,n\KL(\hat p_n, p)>\log(1/\delta)+ (m-1)\log\left(e\left(1+\frac{n}{m-1}\right)\right)\right)\leq \delta .
\end{align*}
\end{lemma}
Next, we state a deviation inequality for adapted Bernoulli random variables, which will be used to control state-action visitation counts. Let $(\mathcal F_t)_{t\in\mathbb{N}}$ be a filtration and let $(X_t)_{t\in\mathbb{N}}$ be Bernoulli random variables such that $X_t$ is $\mathcal F_t$-measurable and
\[
\Pr(X_t=1\mid \mathcal F_{t-1})=P_t,
\]
where $P_t$ is $\mathcal F_{t-1}$-measurable.

\begin{lemma}[{\cite[Lemma~F.4]{dann2017unifying}}]
\label{lemma:bernoulli-deviation}
For all $\delta>0$,
\begin{align}
\Pr \left(\exists n\in\mathbb{N}:\,\sum_{t=1}^n X_t<\frac{1}{2}\sum_{t=1}^n P_t-\log\frac{1}{\delta}\right)\leq \delta .
\end{align}
\end{lemma}

\begin{lemma}
\label{lem:log_threshold}
Let $c\ge 0$ and $b\ge 1$, and suppose $f:\mathbb N\to\mathbb R_+$ satisfies $f(m)\le c+b+b\log(m+1)$ for all $m\in\mathbb N$. Then, for any $\kappa>0$, define
\[
m^\dagger := 1+\frac{2c}{\kappa}+\frac{2b}{\kappa}\log\!\left(\max\left\{\frac{2b}{\kappa},e\right\}\right).
\]
Then $m\ge m^\dagger$ implies $m\ge f(m)/\kappa$.
\end{lemma}

\begin{proof}
Let $t:=\max\{2b/\kappa,e\}$. By concavity of the logarithm, for all $z,t>0$,
\[
\log z \le \log t+\frac{z-t}{t} = \frac{z}{t}+\log t-1.
\]
Applying the above inequality with $z=m+1$ and multiplying by $b$ gives
\[
b\log(m+1)\le \frac{b(m+1)}{t}+b\log t-b.
\]
Since $t\ge 2b/\kappa$, we have $b/t\le \kappa/2$, and hence
\[
b\log(m+1)\le \frac{\kappa(m+1)}{2}+b\log t-b.
\]
By definition of $f(m)$, we obtain
\[
f(m)\le c+b+b\log(m+1)\le c+\frac{\kappa(m+1)}{2}+b\log t.
\]
Therefore, $\kappa m\ge f(m)$ holds whenever
\[
\kappa m\ge c+\frac{\kappa(m+1)}{2}+b\log t.
\]
The last condition is equivalent to
\[
m\ge 1+\frac{2c}{\kappa}+\frac{2b}{\kappa}\log t = 1+\frac{2c}{\kappa}+\frac{2b}{\kappa}\log\!\left(\max\left\{\frac{2b}{\kappa},e\right\}\right)=m^\dagger.
\]
Thus $m\ge m^\dagger$ implies $\kappa m\ge f(m)$, equivalently $m\ge f(m)/\kappa$.
\end{proof}

\begin{lemma}\label{lem:donsker-varadhan}
Let \(\mc X\) be a finite set, let \(p,q\in\Delta(\mc X)\) satisfy \(\KL(q,p)<\infty\), and let \(w:\mc X\to\mathbb R\). Then, for every \(\lambda>0\),
\[
\sum_{x\in\mc X}q(x)w(x)\leq\frac{1}{\lambda}\left(\KL(q,p)+\log\left(\sum_{x\in\mc X}p(x)e^{\lambda w(x)}\right)\right).
\]
\end{lemma}

\begin{proof}[Proof of Lemma~\ref{lem:donsker-varadhan}]
Let
\[
Z\coloneqq\sum_{x\in\mc X}p(x)e^{\lambda w(x)}
\qquad\text{and}\qquad
p_\lambda(x)\coloneqq\frac{p(x)e^{\lambda w(x)}}{Z}.
\]
Since \(\KL(q,p)<\infty\), the support of \(q\) is contained in that of \(p\), which coincides with the support of \(p_\lambda\). Therefore,
\[
0\leq\KL(q,p_\lambda)=\sum_{x\in\mc X}q(x)\log\left(\frac{q(x)Z}{p(x)e^{\lambda w(x)}}\right)=\KL(q,p)-\lambda\sum_{x\in\mc X}q(x)w(x)+\log Z.
\]
Rearranging the above inequality completes the proof.
\end{proof}


\begin{lemma}\label{lem:beta-monotone}
Let $\beta(z,\delta)=\log(2|\mc S||\mc A|H/\delta)+|\mc S|\log(8e(z+1))$ with
$\delta\in(0,1)$. Then
\begin{enumerate}
\item[(i)] $z\mapsto\beta(z,\delta)$ is nondecreasing on $[0,\infty)$ and
$\beta(z,\delta)\ge1$ for every $z\ge0$;
\item[(ii)] $z\mapsto\beta(z,\delta)/z$ is nonincreasing on $(0,\infty)$;
equivalently, $z\mapsto z/\beta(z,\delta)$ is nondecreasing on $[0,\infty)$.
\end{enumerate}
\end{lemma}
\begin{proof}[Proof of Lemma~\ref{lem:beta-monotone}]
For part~(i), $\beta(z,\delta)$ is nondecreasing in \(z\) follows from the fact that \(\log(8e(z+1))\) is nondecreasing in \(z\).  Moreover, since \(|\mc S||\mc A|H\geq1\), \(\delta<1\), and \(|\mc S|\geq1\),
\[
\beta(z,\delta)\geq\log 2+\log(8e)>1.
\]

For part~(ii), define \(\varphi(z):=\beta(z,\delta)/z\) for \(z>0\). Its derivative is
\[
\varphi'(z)=\frac{z|\mc S|-(z+1)\beta(z,\delta)}{(z+1)z^2}\le \frac{|\mc S|-\beta(z,\delta)}{z^2}\le 0.
\]
Hence, \(\varphi(z)\) is nonincreasing on \((0,\infty)\). Equivalently, \(z\mapsto z/\beta(z,\delta)\) is nondecreasing
on \([0,\infty)\), with value \(0\) at \(z=0\), since \(\varphi(z)=\beta(z,\delta)/z\to+\infty\)
as \(z\downarrow0\)
\end{proof}

\begin{lemma}\label{lem:radius-counting}
For real $z\ge0$, define
\[
q(z)=\min\left\{1,\sqrt{\frac{4\beta(z,\delta)}{z\vee1}}\right\},
\qquad
\Phi(z)=\beta(z,\delta)\left[1+\log_+\frac z{4\beta(z,\delta)}\right],
\]
where $\beta(z,\delta)=\log(2|\mc S||\mc A|H/\delta)+|\mc S|\log(8e(z+1))$. Then
\begin{enumerate}
\item[(i)] the functions $z\mapsto\sqrt{\beta(z,\delta)z}$ and
$z\mapsto\Phi(z)$ are nondecreasing on $[0,\infty)$;
\item[(ii)] for every $m\in\mathbb N$ and all $v_1,\dots,v_m\in[0,1]$, setting
$U_i:=\sum_{j<i}v_j$ and $M:=\sum_{i=1}^mv_i$,
\[
\sum_{i=1}^mv_i\,q(U_i)\le4\sqrt2\sqrt{\beta(M,\delta)M},
\qquad
\sum_{i=1}^mv_i\,q(U_i)^2\le8\Phi(M).
\]
\end{enumerate}
\end{lemma}
\begin{proof}[Proof of Lemma~\ref{lem:radius-counting}]
\emph{Part (i).} The map \(z\mapsto\sqrt{\beta(z,\delta)z}\) is the product of two nonnegative, nondecreasing functions and is therefore nondecreasing by (i) from Lemma~\ref{lem:beta-monotone}. 

To analyze \(\Phi\), define \(\psi(z):=z/(4\beta(z,\delta))\), which is nondecreasing by (ii) from Lemma~\ref{lem:beta-monotone}. On the region where \(\psi(z)>1\),
\(
\Phi(z)=\beta(z,\delta)\bigl(1+\log\psi(z)\bigr),
\)
which is also nondecreasing because both factors are nonnegative and nondecreasing. The two expressions agree when \(\psi(z)=1\), so \(\Phi\) is nondecreasing on \([0,\infty)\).

\emph{Part (ii).}
If \(M=0\), then \(v_i=0\) for every \(i\), and both bounds holds immediately. Assume \(M>0\), and define
\[
a:=4\beta(M,\delta)\geq4,\qquad f(z):=\min\left\{1,\sqrt{\frac{a}{z}}\right\}\ \text{for }z>0,\qquad f(0):=1.
\]
The function \(f\) is nonincreasing. We use it as a upper bound for \(q\).

\emph{Claim 1.} For every \(i\), \(q(U_i)\leq f(U_i)\).

If \(U_i=0\), both sides equal \(1\). If \(U_i>0\), then \(U_i\leq M\) and (i) from Lemma~\ref{lem:beta-monotone} give
\[
\frac{4\beta(U_i,\delta)}{U_i\vee1}\leq\frac{4\beta(M,\delta)}{U_i}=\frac{a}{U_i}.
\]
The claim follows from the definitions of \(q\) and \(f\).

\emph{Claim 2.} For every \(z\in[U_i,U_i+v_i]\),
\[
f(U_i)\leq\sqrt2\,f(z).
\]

If \(U_i<a\), then \(f(U_i)=1\). Since \(v_i\leq1\),
\[
z\leq U_i+v_i<a+1,
\]
and hence
\[
f(z)\geq\sqrt{\frac{a}{a+1}}\geq\sqrt{\frac45}\geq\frac{1}{\sqrt2},
\]
where we used \(a\geq4\). If \(U_i\geq a\), then \(z\geq a\), and
\[
\frac{f(U_i)}{f(z)}=\sqrt{\frac{z}{U_i}}\leq\sqrt{1+\frac{v_i}{U_i}}\leq\sqrt{1+\frac1a}\leq\sqrt{\frac54}\leq\sqrt2.
\]
Thus, the claim holds in both cases. Taking the square of both sides also gives \(f(U_i)^2\leq2f(z)^2\).

We now convert the sums into integrals. Since \(U_1=0\) and \(U_{i+1}=U_i+v_i\), the intervals \([U_i,U_i+v_i]\) cover \([0,M]\) consecutively. Claims 1 and 2 imply
\[
v_iq(U_i)\leq\int_{U_i}^{U_i+v_i}\sqrt2\,f(z)\,\mathrm dz,\qquad v_iq(U_i)^2\leq\int_{U_i}^{U_i+v_i}2f(z)^2\,\mathrm dz.
\]
Summing over \(i\) yields
\[
\sum_{i=1}^m v_iq(U_i)\leq\sqrt2\int_0^M f(z)\,\mathrm dz,\qquad \sum_{i=1}^m v_iq(U_i)^2\leq2\int_0^M f(z)^2\,\mathrm dz.
\]

For the first integral, if \(M\leq a\), then \(f(z)=1\) on \([0,M]\), so
\[
\int_0^M f(z)\,\mathrm dz=M\leq\sqrt{aM}.
\]
If \(M>a\), then
\[
\int_0^M f(z)\,\mathrm dz=a+\int_a^M\sqrt{\frac{a}{z}}\,\mathrm dz=2\sqrt{aM}-a\leq2\sqrt{aM}.
\]
Therefore,
\[
\sum_{i=1}^m v_iq(U_i)\leq2\sqrt{2aM}=4\sqrt2\sqrt{\beta(M,\delta)M}.
\]

For the second integral, if \(M\leq a\), then \(\int_0^M f(z)^2\,\mathrm dz=M\leq a\). If \(M>a\), then
\[
\int_0^M f(z)^2\,\mathrm dz=a+\int_a^M\frac{a}{z}\,\mathrm dz=a\left(1+\log\frac{M}{a}\right).
\]
Thus, in both cases,
\[
\int_0^M f(z)^2\,\mathrm dz\leq a\left(1+\log_+\frac{M}{a}\right).
\]
Using \(a=4\beta(M,\delta)\), we conclude that
\[
\sum_{i=1}^m v_iq(U_i)^2\leq2a\left(1+\log_+\frac{M}{a}\right)=8\beta(M,\delta)\left(1+\log_+\frac{M}{4\beta(M,\delta)}\right)=8\Phi(M).
\]
\end{proof}

\begin{lemma}\label{lem:self-bounding}
Let \(C\geq1\) and suppose that \(T\geq0\) satisfies
\(
T\leq C\log^2\bigl(e(T+1)\bigr).
\)
Then,
\[
T\leq64C\log^2(64eC).
\]
\end{lemma}

\begin{proof}[Proof of Lemma~\ref{lem:self-bounding}]
Define
\[
g(z)\coloneqq\frac{z}{\log^2(e(z+1))},\qquad z\geq0.
\]
We first show that \(g\) is increasing. Indeed,
\[
g'(z)=\frac{1}{(1+\log(1+z))^3}\left(1+\log(1+z)-\frac{2z}{1+z}\right).
\]
Note that the expression in parentheses is minimized at \(z=1\), where it equals \(\log2>0\). Hence, \(g'(z)>0\) for all \(z\geq0\).

Now set
\[
x\coloneqq64eC,\qquad L\coloneqq\log x,\qquad T_0\coloneqq64CL^2.
\]
We next show that \(g(T_0)\geq C\). Since \(x\geq64e\), we have \(L^2\leq x\) and \(e\leq x\). Therefore,
\[
e(T_0+1)=xL^2+e\leq x^2+x\leq x^8.
\]
Taking logarithms gives
\[
\log\bigl(e(T_0+1)\bigr)\leq8\log x=8L.
\]
Consequently,
\[
g(T_0)=\frac{T_0}{\log^2(e(T_0+1))}\geq\frac{64CL^2}{64L^2}=C.
\]

The hypothesis $T\le C\log^2(e(T+1))$ gives $g(T)\le C$. If \(T>T_0\), the monotonicity of \(g\) would imply \(g(T)>g(T_0)\geq C\), which is a contradiction. Therefore,
\(
T\leq T_0=64C\log^2(64eC).
\)
\end{proof}

\end{document}